\documentclass[12pt,draftclsnofoot,onecolumn]{IEEEtran}
\pdfoutput=1
\usepackage{amsmath}
\usepackage[backref=page, hidelinks]{hyperref}
\usepackage{graphicx}
\usepackage{color}

\usepackage{cite}
\usepackage{amsthm}
\usepackage{amssymb,amsfonts}
\usepackage{algorithmic}
\usepackage{textcomp}
\usepackage{float}
\usepackage[mathscr]{eucal}
\usepackage{ifpdf}
\usepackage{units}
\usepackage{setspace}
\usepackage{algorithm}
\usepackage{array}
\usepackage{xcolor}
\usepackage{multirow}
\usepackage{enumerate}
\usepackage{mathtools}
\usepackage{footnote}
\newtheorem{theorem}{Theorem}
\newtheorem{lemma}{Lemma}

\usepackage{graphics}
\usepackage{bbm}

\usepackage{tikz}
\usetikzlibrary{arrows,calc}
\DeclareMathOperator*{\argmaxA}{arg\,max}

\usepackage{scalerel}
\usepackage{subcaption}
\usepackage{wrapfig,booktabs}
\newcommand{\norm}[1]{\left\lVert#1\right\rVert}
\allowdisplaybreaks
\begin{document}
%-------------------------------------> Title
\title{Online Learning with Costly Features in Non-stationary Environments
% \thanks{This research was supported by Grant 01IS20051 from the German Federal Ministry of Education and Research (BMBF).}
}
%-------------------------------------> Author
\author{
\IEEEauthorblockN{Saeed Ghoorchian, Evgenii Kortukov, and Setareh Maghsudi}
% First Author\inst{1}\orcidID{0000-1111-2222-3333} \and
% Third Author\inst{3}\orcidID{2222--3333-4444-5555}
\thanks{The authors are with the Faculty of Mathematics and Natural Sciences, T{\"u}bingen University, 72074 T{\"u}bingen, Germany. S. M. is also with the Fraunhofer Heinrich Herz Institute, Berlin, Germany. 
E-mail: saeed.ghoorchian@uni-tuebingen.de, evgenii.kortukov@student.uni-tuebingen.de, setareh.maghsudi@uni-tuebingen.de
}
}
%
% \authorrunning{S. Ghoorchian et al.}
% \institute{
% Eberhard Karls University of T{\"u}bingen, T{\"u}bingen, Germany
% \email{saeed.ghoorchian@uni-tuebingen.de} \\
% \email{evgenii.kortukov@student.uni-tuebingen.de} \\
% \email{setareh.maghsudi@uni-tuebingen.de}
% Princeton University, Princeton NJ 08544, USA \and
\maketitle % typeset the header of the contribution
%
%-------------------------------------> Abstract
\begin{abstract}
Maximizing long-term rewards is the primary goal in sequential decision-making problems. The majority of existing methods assume that side information is freely available, enabling the learning agent to observe all features' states before making a decision. In real-world problems, however, collecting beneficial information is often costly. That implies that, besides individual arms' reward, learning the observations of the features' states is essential to improve the decision-making strategy. The problem is aggravated in a non-stationary environment where reward and cost distributions undergo abrupt changes over time. To address the aforementioned dual learning problem, we extend the contextual bandit setting and allow the agent to observe subsets of features' states. The objective is to maximize the long-term average gain, which is the difference between the accumulated rewards and the paid costs on average. Therefore, the agent faces a trade-off between minimizing the cost of information acquisition and possibly improving the decision-making process using the obtained information. To this end, we develop an algorithm that guarantees a sublinear regret in time. Numerical results demonstrate the superiority of our proposed policy in a real-world scenario.
% \keywords{Contextual multi-armed bandit  \and non-stationary process \and online learning \and costly information acquisition.}
\end{abstract}
%-------------------------------------> Keywords
% \begin{IEEEkeywords}
\section{Introduction}
\label{sec:Intro}
In a sequential decision-making problem, an agent takes action over consecutive rounds of play to optimize a long-term metric. Over the past decades, a large body of literature develop decision-making policies that deal with such optimization problems under various constraints \cite{Lattimore20:BAL, Hoi21:OLA}. In most cases, particularly in the era of big data, the proposed methods postulate the possibility of information acquisition with no limit and for free. In reality, however, access to side information is challenging; collecting information might be costly. For example, in online advertising problems, the advertiser can purchase information about target users to display personalized ads. As another example, in medical contexts, obtaining information for treatment recommendations mainly requires additional tests that are time- and money-consuming. Thus, it is essential to develop algorithms that can learn the optimal observations and actions simultaneously.

Real-world problems frequently appear in non-stationary environments. For instance, in the application of personalized news recommendation, user preferences over news can change over time and exhibit various seasonality patterns \cite{Wu19:DEC}. As another example, in the wireless network routing problem, the quality and availability of each link may change over time due to network congestion or maintenance \cite{Liu12:ASP}. The dual learning problem described above becomes significantly more challenging when the environment changes. In fact, in a non-stationary environment, the value of obtained information, such as received action's feedback or paid observation's cost, before a change in the environment might become obsolete after the change occurs. Therefore, the agent has to constantly adapt her strategy and improve the decision-making process to comply faster with the changes in the environment, while she simultaneously performs the aforementioned dual learning task.

We address the mentioned challenges by using the Multi-Armed Bandit (MAB) \cite{Robbins52:SAS} framework, where a learning agent selects an arm at sequential decision-making rounds and the environment reveals a feedback drawn from some unknown probability distribution. In this setting, the agent experiences the exploration-exploitation dilemma, where the decision has to be made between exploring options to acquire new knowledge and selecting an option by exploiting the existing knowledge \cite{Maghsudi16:MAB}. In a contextual MAB problem, the agent has additional access to some side information and is able to observe this contextual information before making decision at each round. However, in practice, such contextual information is not always readily available to the agent, but rather it has to be acquired in exchange for a cost. 

In this paper, we model the described problem using the contextual bandit setting and introduce the non-stationary costly contextual bandit problem, which we call it NCC problem for short. We propose and analyze an algorithm to solve the NCC problem. Our proposed algorithm can be considered as a variant of the UCRL2 algorithm \cite{Jaksch10:NOR}. Moreover, it uses a sliding window to estimate the non-stationary rewards and costs. We prove that our algorithm achieves a sublinear regret bound in time. We validate our solution on a real-world problem of ranking nursery school applications. The results demonstrate the superiority of our algorithm compared to several benchmarks.
%-------------------------------------> Subsection Related Works
\subsection{Related Works}
\label{subsec:RelatedWorks}
%
% The exploration-exploitation dilemma is ubiquitous in real-world problems. Potential application domains span across different fields, including online recommender systems \cite{Xu20:CBB}, edge computing problems \cite{Ghoorchian21:MAB}, design of clinical trials \cite{Villar15:MBM}, or targeted Covid-19 border testing of travelers \cite{Bastani21:EAT}. In MAB problems, such dilemma is well-addressed in stationary environments \cite{Auer02:FTA}, \cite{Li10:ACB}, \cite{Agrawal13:TSF}, \cite{Langford07:TEG}, \cite{Agarwal14:TTM}. Different strategies, such as those based on Upper Confidence Bounds (UCBs) \cite{Li10:ACB} and Thompson sampling \cite{Thompson33:OTL} as well as greedy approaches \cite{Langford07:TEG}, are proposed in the literature. Despite achieving promising results in stationary environments, they fail in non-stationary environments.

% This non-stationary setting is considered in \cite{Hariri15:ATU} to develop a recommendation system with abruptly changing user preferences.
 % Several works examine the application of linear non-stationary contextual MAB problems to a recommendation system scenario \cite{Zeng16:OCA, Wu18:LCB, Xu20:CBB}.
Non-stationary multi-armed bandits have attracted intensive attention in the past years, both from the theory \cite{Garivier11:OUC, Luo18:ECB, Chen19:ANA, Russac19:WLB, Cheung19:LTO} and the application \cite{Hariri15:ATU, Zeng16:OCA, Wu18:LCB, Xu20:CBB} side. Potential application domains span across different fields, including online recommender systems \cite{Zeng16:OCA, Wu18:LCB, Wu19:DEC, Xu20:CBB}, edge computing problems \cite{Ghoorchian21:MAB}, hyperparameter optimization \cite{Lu22:NSC}, virtual reality for rehabilitation \cite{Kamikokuryo22:AAA}, split liver transplantation allocation \cite{Tang21:MAB}, evaluation of information retrieval systems \cite{Losada17:MAB}, or targeted Covid-19 border testing of travelers \cite{Bastani21:EAT}. The state-of-the-art methods in non-stationary bandits either do not consider access to contextual information or do not assume costly information acquisition. In the seminal work of \cite{Garivier11:OUC}, the authors use a sliding window or a discount factor to estimate the rewards with piece-wise stationary generating processes. \cite{Cheung19:LTO} studies the linear stochastic bandit in a drifting environment with a variation budget. The authors propose an Upper Confidence Bound (UCB)-based algorithm that adapts to reward changes using a sliding window and a Bandit-over-Bandit framework for tuning the proposed algorithm's parameter adaptively. The authors in \cite{Russac19:WLB} study linear stochastic bandits in abruptly changing and slowly varying environments. They utilize exponentially increasing weights of observations to reduce the influence of past observations with time, thereby adapting to environmental changes. In \cite{Xu20:CBB}, the authors consider a contextual bandit problem and use two sliding windows to detect changes in reward distributions. If the rewards inside the second window are not predictable with high accuracy from observations inside the first window, the proposed algorithm considers a new change point. The observations since the last change point are used to select arms. Besides, \cite{Zeng16:OCA} uses Gaussian random walks to model the non-stationarity in underlying reward-generating processes. Online inference based on particle learning is applied to fit the bandit parameters sequentially. Moreover, \cite{Wu18:LCB} proposes a hierarchical bandit algorithm, which maintains a suite of bandit models that estimate the reward distributions using a subset of observations. A higher level bandit model measures if the prediction error of lower level models exceeds some threshold, discards them accordingly, and creates new ones. Further, \cite{Luo18:ECB, Chen19:ANA} study the general non-stationary contextual MAB problem and propose algorithms that achieve sublinear regret bounds without the knowledge of the number of change points. However, both of these works do not consider costly information acquisition. Our paper, in contrast, focuses on non-stationary contextual bandits with general (linear or nonlinear) reward and cost functions. Our proposed algorithm achieves sublinear regret by adapting to reward and cost distribution drifts, conditioned on tuning the sliding window size.

% \textcolor{green}{In \cite{Luo18:ECB, Chen19:ANA}, the authors analyze the general non-stationary contextual MAB problem. \cite{Luo18:ECB} proposes algorithms with sublinear regret when the number of change points $S$ or the reward variation $\Delta$ is known. 
% The reward variation is defined as $\Delta = \sum_{t=2}^{T}{\max_{\pi \in \Pi}|\mathcal{R}_{t}(\pi) - \mathcal{R}_{t-1}(\pi)|}$.
% It also introduces an algorithm that works without the knowledge of $S$ or $\Delta$ that achieves sublinear regret. \cite{Chen19:ANA} further improves the developed regret bound in \cite{Luo18:ECB} for the case when $S$ or $\Delta$ are unknown in advance. 
% However, both of these works do not consider costly information acquisition. Our paper, in contrast, focuses on the connection between non-stationarity and costly features. We propose an algorithm that achieves sublinear regret by adapting to changing reward and feature cost distributions, conditioned on tuning the sliding window size.
% Investigating whether a sublinear regret bound is achievable without tuning the parameters in the costly features setting is a possible direction for future work.
% }

Costly features in online learning problems have been addressed both in the full information setting \cite{Zolghadr13:OLW, Janisch20:CWC, Shim18:JAF}, and in the bandit setting \cite{Atan21:DDO}. However, the existing methods with bandit feedback either do not model the cost as a random variable or do not take into account the non-stationarity of the environment. Reference \cite{Atan21:DDO} is the most relevant work to ours. The authors consider a contextual bandit problem where observing features' states is costly. However, the costs are constant values, and the reward-generating processes are stationary. Our approach shall not be mistaken for MAB problems with paid observations \cite{Seldin14:PLA}, where the agent can observe the rewards of any subset of arms after paying the costs at each round. In contrast, in our work, we allow for feature vectors and assume that observing feature's states is costly.

Another related area of research is budget-constrained learning, where feature selection is adaptive. For example, the authors in \cite{Cesa-Bianchi11:ELP} consider linear regression models under local and global constraints on the number of observed features. They propose an algorithm that actively chooses the features to observe for each data sample. As another example, the authors in \cite{Hazan12:LRL} consider linear regression with a budget on the number of feature observations for each data sample. They analyze the number of required samples for the model with partial information to attain the same error as that with complete information. Unlike our approach, these works consider a batch learning setting with the free observation of a limited number of features. Besides, in \cite{Janisch20:CWC}, the authors investigate an online classification problem with a per-sample budget for observing features, where features have various costs. They propose a deep reinforcement learning algorithm to solve the problem. \cite{Bouneffouf17:CAB} studies a contextual bandit problem in which the agent has a fixed budget on the number of features she can observe before choosing an arm. The authors take advantage of Thompson sampling and propose an algorithm that works in stationary and non-stationary environments. However, they do not provide regret analysis for the proposed method. Compared to the aforementioned works, we do not assume a budget constraint; nonetheless, the agent attempts to minimize the total cost of observing features' states. Therefore, in our proposed method, the agent adaptively selects the features and learns the optimal policy from limited information.

%-------------------------------------> Organization
% \subsection{Organization}
% \label{subsec:organization}
%
The rest of the paper is as follows. We formulate the NCC bandit problem in Section \ref{sec:ProFor}. We describe our proposed method, NCC-UCRL2, in Section \ref{sec:strategy}. In Section \ref{sec:regretanalysis}, we analyze the performance of NCC-UCRL2 theoretically. Section \ref{sec:NumAnalysis} includes numerical evaluation, and Section \ref{sec:Conclusion} concludes the paper.
%-------------------------------------> Section Pro For
\section{Problem Formulation}
\label{sec:ProFor}
Let $\mathcal{A} = \{1, 2, \ldots, A\}$ denote the set of \textit{actions}. $\mathcal{D} = \{1, 2, \dots, D\}$ represents a finite set of \textit{features}. Each feature $i \in \mathcal{D}$ has some random state $\boldsymbol{\Phi}[i] \in \mathcal{X}_{i}$, where $\mathcal{X}_{i}$ denotes a finite set of states for feature $i$. We collect the random features' states of all the features in the random state vector $\boldsymbol{\Phi} = \left[\boldsymbol{\Phi}[1], \boldsymbol{\Phi}[2], \ldots, \boldsymbol{\Phi}[D] \right] \in \mathcal{X} = \bigotimes_{i \in \mathcal{D}} \mathcal{X}_{i}$. Let $\boldsymbol{\phi}$ be a realization of the random state vector, which is drawn from a fixed but unknown distribution. $\mathbb{P}[\boldsymbol{\Phi} = \boldsymbol{\phi}]$ shows the probability of state vector $\boldsymbol{\phi}$ being realized. 

At each time $t$, the environment draws a \textit{state vector} $\boldsymbol{\phi}_{t} = [\boldsymbol{\phi}_{t}[1], \boldsymbol{\phi}_{t}[2], \dots,  \newline \boldsymbol{\phi}_{t}[D] ]$. The agent can select a subset of features $\mathcal{I}_{t} \subseteq \mathcal{D}$, called the \textit{observation set}, for costly observation. Other elements of the state vector remain unknown.  When $|\mathcal{I}_{t}| = 0$, i.e., $\mathcal{I}_{t} = \emptyset$, none of features' states are observed at time $t$. We use $\mathcal{P}(\mathcal{D})$ to represent the power set of $\mathcal{D}$ that includes all possible observation sets, i.e., $\mathcal{P}(\mathcal{D}) = \{\mathcal{I} \subseteq \mathcal{D}~|~0 \leq |\mathcal{I}| \leq D\}$. Besides, the \textit{partial state vector} $\boldsymbol{\psi}_{t} = \left[\boldsymbol{\psi}_{t}[1], \boldsymbol{\psi}_{t}[2], \ldots, \boldsymbol{\psi}_{t}[D]\right]$ can be represented as
\begin{align}
\label{eq:ParStateDef}
\boldsymbol{\psi}_{t}[i] = \begin{cases}
\boldsymbol{\phi}_{t}[i], & \quad \text{if}~~i \in \mathcal{I}_{t}, \\
\textup{N/A}, & \quad \text{if}~~i \notin \mathcal{I}_{t},
\end{cases}
\end{align}
where $\textup{N/A}$ indicates the corresponding feature's state is missing. Let $\mathscr{D}(\boldsymbol{\psi}) = \{i \in \mathcal{D}~|~\boldsymbol{\psi}[i] \neq \; \textup{N/A}\}$ represent the \textit{domain set} of a partial state vector $\boldsymbol{\psi}$. By $\boldsymbol{\Psi}^{+}(\mathcal{I}) = \{\boldsymbol{\psi}~|~\mathscr{D}(\boldsymbol{\psi}) = \mathcal{I}\}$, we denote the set of all possible partial state vectors whose domain set is equal to the observation set $\mathcal{I}$. Therefore, $\boldsymbol{\Psi} = \bigcup_{\mathcal{I} \subseteq \mathcal{D}} \boldsymbol{\Psi}^{+}(\mathcal{I})$ denotes the set of all possible partial state vectors. Furthermore, we define a partial state vector $\boldsymbol{\psi}$ to be \textit{consistent} with $\boldsymbol{\phi}$ if $\boldsymbol{\psi}[i] = \boldsymbol{\phi}[i]$, $\forall i \in \mathscr{D}(\boldsymbol{\psi})$. We use $\boldsymbol{\phi} \sim \boldsymbol{\psi}$ to show that $\boldsymbol{\psi}$ is consistent with $\boldsymbol{\phi}$. Moreover, $\boldsymbol{\psi}$ is a \textit{substate} of $\boldsymbol{\psi}'$ if both the partial state vectors $\boldsymbol{\psi}$ and $\boldsymbol{\psi}'$ are consistent with $\boldsymbol{\phi}$ and $\mathscr{D}(\boldsymbol{\psi}) \subseteq \mathscr{D}(\boldsymbol{\psi}')$. We use $\boldsymbol{\psi} \preceq \boldsymbol{\psi}'$ to show that $\boldsymbol{\psi}$ is a substate of $\boldsymbol{\psi}'$. For every $i \in \mathcal{I}_{t}$, $\boldsymbol{c}_{t}[i] \in [0,1]$ shows the random cost to observe $\boldsymbol{\phi}_{t}[i]$, which follows an unknown probability distribution with mean $\bar{\boldsymbol{c}}_{t}[i]$. Also, by $\boldsymbol{c}_{t} = [\boldsymbol{c}_{t}[1], \boldsymbol{c}_{t}[2], \dots, \boldsymbol{c}_{t}[D]]$ and $\bar{\boldsymbol{c}}_{t} = \bar{\boldsymbol{c}}_{t}[1], \bar{\boldsymbol{c}}_{t}[2], \dots, \bar{\boldsymbol{c}}_{t}[D]]$, we denote the \textit{cost vector} and the \textit{mean cost vector} of all features at time $t$, respectively. 

At each time $t$, the agent follows a policy $\pi_{t}$ to select an observation set $\mathcal{I}_{t}$ and an action $a_{t}$. Therefore, we define the \textit{policy} at time $t$ using a tuple $\pi_{t} = (\mathcal{I}_{t}, h_{t})$, where $h_{t}: \boldsymbol{\Psi}^{+}(\mathcal{I}_{t}) \rightarrow \mathcal{A}$ denotes an adaptive action selection strategy that maps a partial state vector $\boldsymbol{\psi}_{t} \in \boldsymbol{\Psi}^{+}(\mathcal{I}_{t})$ to an action $a_{t} \in \mathcal{A}$. The agent then receives a random reward $r_{t} \in [0,1]$ whose distribution is unknown a priori. We define the unknown expected reward function as $\bar{r}_{t}:\mathcal{A} \times \mathcal{X} \rightarrow \left[0,1\right]$; hence $\bar{r}_{t}(a_{t},\boldsymbol{\phi}_{t})$ is the expected reward of action $a_{t}$ at time $t$ when the state vector is $\boldsymbol{\phi}_{t}$. The generating processes of rewards and costs are piece-wise stationary so that there exist $\Upsilon_{T}$ time instants before a time horizon $T$ where at least one of the mean rewards or mean costs changes abruptly. We define the marginal probabilities and expected rewards of partial state vectors using the definition of probability distribution and expected reward for the state vectors. The marginal probability of the partial state vector $\boldsymbol{\psi}_{t}$ being realized at time $t$ is defined as $p(\boldsymbol{\psi}_{t}) = \mathbb{P}[\boldsymbol{\Phi}_{t} \sim \boldsymbol{\psi}_{t}]$. Moreover, $\bar{r}_{t}(a_{t},\boldsymbol{\psi}_{t}) = \mathbb{E}\left[ \bar{r}_{t}(a_{t},\boldsymbol{\Phi}_{t})~|~\boldsymbol{\Phi}_{t} \sim \boldsymbol{\psi}_{t} \right]$ indicates the marginal expected reward of action $a_{t}$ when the partial state vector $\boldsymbol{\psi}_{t}$ is observed. Therefore, for a fixed observation set $\mathcal{I}$, it holds that $\sum_{\boldsymbol{\psi} \in \boldsymbol{\Psi}^{+}(\mathcal{I})} p(\boldsymbol{\psi}) = 1$.

The \textit{expected gain} of the agent following the policy $\pi=(\mathcal{I},h)$ at time $t$ yields
\begin{align}
\label{eq:ExpGain}
    \rho_{t}^{\pi} = \sum_{\boldsymbol{\psi} \in \boldsymbol{\Psi}^{+}(\mathcal{I})} p(\boldsymbol{\psi}) \bar{r}_{t}(h(\boldsymbol{\psi}), \boldsymbol{\psi}) - \sum_{i \in \mathcal{I}} \bar{\boldsymbol{c}}_{t}[i].
\end{align}
In words, the expected gain of the agent that follows a policy $\pi$ at time $t$ is the expected reward of $\pi$ received by the agent at time $t$ minus the expected cost of $\pi$ incurred by the agent due to state observation at time $t$. Let $\Pi$ denote the set of all feasible policies defined as
\begin{align}
    \Pi = \{ (\mathcal{I}, h) | \mathcal{I} \in \mathcal{P}(\mathcal{D}) \}.
\end{align}
Therefore, the optimal policy $\pi_{t}^{\ast} = (\mathcal{I}_{t}^{\ast}, h_{t}^{\ast})$ at time $t$ is given by
\begin{align}
\label{eq:optimalpolicy}
    \pi_{t}^{\ast} = \argmaxA_{\pi \in \Pi}~~\rho_{t}^{\pi}.
\end{align}
Moreover, the expected gain of the optimal policy at time $t$ is denoted by $\rho_{t}^{\ast} = \rho_{t}^{\pi_{t}^{\ast}}$. We summarize the most important notations in \textbf{Table \ref{table:Notations}}.
%-------------------------------------> Table
\renewcommand{\arraystretch}{1.1}
\renewcommand{\tabcolsep}{1.1mm}
\begin{table}[t]
\caption{Summary of Notations.}
\label{table:Notations}
% {\footnotesize
% \centering
\begin{center}
\scalebox{0.85}{
% \resizebox{\columnwidth}{!}{%
\begin{tabular}{c|l}
    % \cline{1-1}
    \hline
    Notation & \multicolumn{1}{c}{Definition}  \\ 
    \hline
    $\mathcal{A}$ & Set of actions \\
    \hline
    $\mathcal{D}$ & Set of features \\
    \hline
    $\boldsymbol{\phi}_{t}$ & Unknown state vector at time $t$ \\
    \hline
    $\mathcal{I}_{t}$ & Observation set of selected features at time $t$ \\
    \hline
    $\boldsymbol{\psi}_{t}$ & Partial state vector observed by the agent at time $t$ \\
    \hline
    $a_{t}$ & Action of the agent at time $t$ \\
    \hline
    $r_{t}$ & Reward at time $t$ \\
    \hline
    % \bar{\boldsymbol{c}}_{t}[i]
    $\boldsymbol{c}_{t}[i]$ & Cost of state observation for feature $i \in \mathcal{D}$ at time $t$ \\
    \hline
    $\rho_{t}^{\pi}$ & Expected gain of policy $\pi$ \\
    \hline
    $\mathscr{D}(\boldsymbol{\psi})$ & Domain set of partial state vector $\boldsymbol{\psi}$ \\
    \hline
    $\boldsymbol{\Psi}^{+}(\mathcal{I})$ & Set of all partial state vectors with domain $\mathcal{I}$ \\
    \hline
    $\boldsymbol{\Psi}$ & Set of all partial state vectors \\
    \hline
    \end{tabular}
}
\end{center}
% }
\end{table}
%--------------------------------

% defined in 
The optimal policy (\ref{eq:optimalpolicy}) for NCC problem differs from the conventional optimal policies in the contextual bandit problems. Let $a_{t}^{\ast}(\boldsymbol{\psi}) = \argmaxA_{a \in \mathcal{A}}~\bar{r}_{t}(a, \boldsymbol{\psi})$ denote the best action for a given partial state vector $\boldsymbol{\psi}$. Moreover, define $\bar{r}_{t}^{\ast}(\boldsymbol{\psi}) = \bar{r}_{t}(a_{t}^{\ast}(\boldsymbol{\psi}), \boldsymbol{\psi})$ as the expected reward of the best action when the partial state vector is $\boldsymbol{\psi}$. Moreover, for a fixed observation set $\mathcal{I}$, define a policy $\pi_{t}(\mathcal{I}) = (\mathcal{I}, a_{t}^{*}(\boldsymbol{\psi}))$ that selects the observation set $\mathcal{I}$ and the best action $a_{t}^{*}(\boldsymbol{\psi})$ for any $\boldsymbol{\psi} \in \boldsymbol{\Psi}^{+}(\mathcal{I})$ at time $t$. The expected gain of the policy $\pi_{t}(\mathcal{I})$ can be calculated as $V_{t}(\mathcal{I}) = \sum_{\boldsymbol{\psi} \in \boldsymbol{\Psi}^{+}(\mathcal{I})} p(\boldsymbol{\psi}) \bar{r}_{t}^{*}(\boldsymbol{\psi}) - \sum_{i \in \mathcal{I}} \bar{\boldsymbol{c}}_{t}[i]$. Then, the optimal policy $\pi_{t}^{\ast} = (\mathcal{I}_{t}^{\ast}, h_{t}^{\ast})$ defined in (\ref{eq:optimalpolicy}) can be obtained by
\begin{equation}
\begin{aligned}
\label{eq:oracle}
    \mathcal{I}_{t}^{\ast} &= \argmaxA_{\mathcal{I} \in \mathcal{P}(\mathcal{D})}~V_{t}(\mathcal{I}), \\
    h_{t}^{\ast}(\boldsymbol{\psi}) &= \argmaxA_{a \in \mathcal{A}}~\bar{r}_{t}(a, \boldsymbol{\psi}).
\end{aligned}
\end{equation}
We observe that $\rho_{t}^{\ast} = V_{t}(\mathcal{I}_{t}^{\ast})$, which means the optimal policy (\ref{eq:optimalpolicy}) achieves the highest expected gain at each time $t$ among all the policies $\pi_{t}(\mathcal{I})$. 

Ideally, the agent aims at maximizing the total expected gain over the time horizon $T$. Alternatively, the agent's goal is to minimize the \textit{expected regret} over the time horizon $T$, defined as the difference between the accumulated expected gain of the oracle that follows the optimal policy and that of the agent that follows the applied policy. Formally, the expected regret is defined as
\begin{align}
\label{eq:ExpRegret}
    \mathcal{R}_{T}(\Pi) = \sum_{t = 1}^{T} \left[\rho_{t}^{\ast} - \rho_{t}^{\pi_{t}}\right].
\end{align}
In the next section, we propose a policy to minimize the expected regret (\ref{eq:ExpRegret}).
%-------------------------------------> Section Warm-Up
\section{Decision-Making Strategy}
\label{sec:strategy}
In this section, we propose our decision-making strategy to solve the NCC problem described in Section \ref{sec:ProFor}. Our policy, presented in \textbf{Algorithm \ref{alg:NCC-UCRL2}}, takes three types of confidence regions into account, for rewards, costs, and probabilities of partial state vectors. Since the random generating processes of rewards and costs are non-stationary, we use a sliding window of size $w > 0$ to estimate their mean values. At each time $t$, we define 
\begin{equation}
\label{eq:CounterSet-a-psi}
\mathcal{T}_{t}(a, \boldsymbol{\psi} ; w) = \{ t - w < \tau < t ~|~ a_{\tau} = a~\&~\boldsymbol{\psi}_{\tau} = \boldsymbol{\psi} \},
\end{equation}
%
% and
%
\begin{align}
\label{eq:CounterSet-i}
    \mathcal{T}_{t}(i ; w) = \{ t - w < \tau < t ~|~ i \in \mathcal{I}_{\tau} \}.
\end{align}
For each $a \in \mathcal{A}$ and $\boldsymbol{\psi} \in \boldsymbol{\Psi}$, we calculate the empirical average of rewards at time $t$ by
\begin{equation}
\label{eq:AveRew}
\hat{r}_{t}(a, \boldsymbol{\psi}) = \frac{1}{N_{t}(a, \boldsymbol{\psi}; w)} \sum_{\tau \in \mathcal{T}_{t}(a, \boldsymbol{\psi}; w)}^{} r_{\tau},
\end{equation}
where $N_{t}(a, \boldsymbol{\psi}; w) = \max\{1, |\mathcal{T}_{t}(a, \boldsymbol{\psi}; w)|\}$. Moreover, at each time $t$, we calculate the empirical average of costs for each $i \in \mathcal{D}$ by
\begin{equation}
\label{eq:AveCost}
\hat{\boldsymbol{c}}_{t}[i] = \frac{1}{N_{t}(i; w)} \sum_{\tau \in \mathcal{T}_{t}(i; w)}^{} \boldsymbol{c}_{\tau}[i],
\end{equation}
where $N_{t}(i; w) = \max\{1, |\mathcal{T}_{t}(i;w)|\}$. 

Our policy uses the collected data to estimate the probabilities of partial state vectors; that is, after observing the partial state vector $\boldsymbol{\psi}_{t}$, the agent uses it to update the estimate of the probability of $\boldsymbol{\psi}_{t}$ and the probabilities of all the substates of $\boldsymbol{\psi}_{t}$. However, the agent cannot use the obtained reward at time $t$ to update the estimate of mean reward for action $a_{t}$ and the sub-states of $\boldsymbol{\psi}_{t}$, since it introduces a bias into the mean reward estimation. Therefore, we define
\begin{equation}
\label{eq:CounterSet-I}
\mathcal{T}_{t}(\mathcal{I}) = \{ \tau < t ~|~ \mathcal{I} \subseteq \mathcal{I}_{\tau} \},
\end{equation}
%
% and
%
\begin{align}
\label{eq:CounterSet-I-psi}
\mathcal{T}_{t}(\mathcal{I}, \boldsymbol{\psi}) = 
\begin{cases}
    \{ \tau < t \hspace{0.4mm}|\hspace{0.4mm} \mathcal{I} \subseteq \mathcal{I}_{\tau} \hspace{0.4mm}\&\hspace{0.4mm} \boldsymbol{\psi} \preceq \boldsymbol{\psi}_{\tau} \}, & \boldsymbol{\psi} \in \boldsymbol{\Psi}^{+}(\mathcal{I}), \\
    \emptyset, & \boldsymbol{\psi} \notin \boldsymbol{\Psi}^{+}(\mathcal{I}).
\end{cases}
\end{align}
Then, we estimate the probability for each partial state vector $\boldsymbol{\psi} \in \boldsymbol{\Psi}$ at time $t$ as
\begin{equation}
\label{eq:EstimatedParStatePro}
\hat{p}_{t}(\boldsymbol{\psi}) = \frac{N_{t}(\mathscr{D}(\boldsymbol{\psi}), \boldsymbol{\psi})}{N_{t}(\mathscr{D}(\boldsymbol{\psi}))},
\end{equation}
where $N_{t}(\mathcal{I}, \boldsymbol{\psi}) = \max\{1, |\mathcal{T}_{t}(\mathcal{I}, \boldsymbol{\psi})|\}$ and $N_{t}(\mathcal{I}) = \max\{1, |\mathcal{T}_{t}(\mathcal{I})|\}$.

%-------------------------------------> Algorithm
% \scalebox{0.9}{
\begin{algorithm}[t!]
\setstretch{1.1}
\caption{NCC-UCRL2}
\label{alg:NCC-UCRL2} 
% Number of arms $A$, w
\textbf{Input:} Window size $w$. %\\
% \vspace{-5mm}
\begin{algorithmic}[1]
\STATE \textbf{Initialize:} $\forall a \in \mathcal{A}$, $\forall \boldsymbol{\psi} \in \boldsymbol{\Psi}$, $\forall i \in \mathcal{D}$, $\forall \mathcal{I} \in \mathcal{P}(\mathcal{D})$:\\
$\mathcal{T}_{1}(a, \boldsymbol{\psi} ; w) = \emptyset$,~$\mathcal{T}_{1}(i ; w) = \emptyset$,~$\mathcal{T}_{1}(\mathcal{I}) = \emptyset$,~$\mathcal{T}_{1}(\mathcal{I}, \boldsymbol{\psi}) = \emptyset$.\\
% $N_{0}(a, \boldsymbol{\psi}; w) = 0$,~$N_{0}(\mathcal{I}) = 0$,~$N_{0}(i; w) = 0$,~$N_{0}(\mathscr{D}(\boldsymbol{\psi}), \boldsymbol{\psi}) = 0$.
%
\FOR{$t = 1, \ldots, T$}
  %
%   \STATE Calculate confidence bounds $C_{t}(a, \boldsymbol{\psi}; w)$, $\forall a \in \mathcal{A}$, $\forall \boldsymbol{\psi} \in \Psi$, $C_{t}(i; w)$, $\forall i \in \mathcal{D}$, $C_{t}(\mathcal{I})$, $\forall \mathcal{I} \in \boldsymbol{\Psi}$.
  %
  \STATE Compute $\hat{r}_{t}(a,\boldsymbol{\psi})$,~$\forall a \in \mathcal{A}$, $\forall \boldsymbol{\psi} \in \boldsymbol{\Psi}$, using (\ref{eq:AveRew}).
  \STATE Compute $\hat{\boldsymbol{c}}_{t}[i]$,~$\forall i \in \mathcal{D}$, using (\ref{eq:AveCost}).
  \STATE Compute $\hat{p}_{t}(\boldsymbol{\psi})$,~$\forall \boldsymbol{\psi} \in \boldsymbol{\Psi}$. using (\ref{eq:EstimatedParStatePro}). 
  \STATE Solve Problem (\ref{eq:optimization}), $\forall \mathcal{I} \in \mathcal{P}(\mathcal{D})$, and obtain $\hat{V}_{t}(\mathcal{I})$.
  \STATE Select the observation set $\hat{\mathcal{I}}_{t}$ that solves (\ref{eq:SelectedObservationSet}) and pay the cost $\sum_{i \in \hat{\mathcal{I}}_{t}}^{} \boldsymbol{c}_{t}[i]$.
  %
  % \vspace{-4.5mm}
  \STATE Determine the action selection strategy $\hat{h}_{t}(\boldsymbol{\psi})$ based on (\ref{eq:DeterminedActionStrategy}).
  \STATE Observe the partial state vector $\boldsymbol{\psi}_{t} \in \boldsymbol{\Psi}^{+}(\hat{\mathcal{I}}_{t})$.
  \STATE Select the action $a_{t} = \hat{h}_{t}(\boldsymbol{\psi}_{t})$ and observe the reward $r_{t}$.
  %
%   \STATE Update $N_{t}(\mathscr{D}(\boldsymbol{\psi}))$ and  $N_{t}(\mathscr{D}(\boldsymbol{\psi}), \boldsymbol{\psi})$, $\forall \boldsymbol{\psi}$~s.t.~$\boldsymbol{\psi} \preceq \boldsymbol{\psi}_{t}$.
%   %
%   \STATE Update $N_{t}(a_{t}, \boldsymbol{\psi}_{t}; w)$.
%   %
%   \STATE Update $N_{t}(i; w)$, $\forall i \in \hat{\mathcal{I}}_{t}$.
%   %
    \STATE Update $\mathcal{T}_{t}(\mathscr{D}(\boldsymbol{\psi}))$ and $\mathcal{T}_{t}(\mathscr{D}(\boldsymbol{\psi}), \boldsymbol{\psi})$, $\forall \boldsymbol{\psi}$~s.t.~$\boldsymbol{\psi} \preceq \boldsymbol{\psi}_{t}$.
    %
    % \vspace{-4.5mm}
    \STATE Update $\mathcal{T}_{t}(a_{t}, \boldsymbol{\psi}_{t}; w)$.
    \STATE Update $\mathcal{T}_{t}(i; w)$, $\forall i \in \hat{\mathcal{I}}_{t}$.
\ENDFOR 
\end{algorithmic} 
\end{algorithm}
% }
%-------------------------------------

When searching for the optimal observation set and action, we add high-probability confidence bounds to the aforementioned estimates. Let $\boldsymbol{\Psi}_{tot} = \sum_{\mathcal{I} \in \mathcal{P}(\mathcal{D})} |\boldsymbol{\Psi}^{+}(\mathcal{I})|$ and $\delta > 0$. For each action $a \in \mathcal{A}$ and partial state vector $\boldsymbol{\psi} \in \boldsymbol{\Psi}$, we define
\begin{align}
    \tilde{r}_{t}(a, \boldsymbol{\psi}) = \hat{r}_{t}(a, \boldsymbol{\psi}) + C_{t}(a, \boldsymbol{\psi}; w),
\end{align}
where $C_{t}(a, \boldsymbol{\psi}; w) = \min\left\{ 1, \sqrt{\frac{\log{(T A \boldsymbol{\Psi}_{tot} w/\delta)}}{N_{t}(a, \boldsymbol{\psi}; w)}} \right\}$. Moreover, for each feature $i \in \mathcal{D}$, we define
\begin{align}
\label{eq:pessimisticcosts}
    \tilde{\boldsymbol{c}}_{t}[i] = \hat{\boldsymbol{c}}_{t}[i] - C_{t}(i; w),
\end{align}
where $ C_{t}(i; w) = \min\left\{ 1, \sqrt{\frac{2 \log{(T D w/\delta)}}{N_{t}(i; w)}} \right\}$. The optimistic gain at time $t$ can be found by searching for partial state vector probabilities over a high-probability space and a policy that solves
\begin{align} %\nonumber
\label{eq:optimization-pi-q}
\underset{\substack{\pi=(\mathcal{I},h), \\q \in \Delta_{|\boldsymbol{\Psi}^{+}(\mathcal{I})|} }}{\text{maximize}} {\Bigg\{} \hspace{-0.5mm} \sum_{\boldsymbol{\psi} \in \boldsymbol{\Psi}^{+}(\mathcal{I})} \hspace{-2mm} q(\boldsymbol{\psi}) \tilde{r}_{t}(h(\boldsymbol{\psi}),\boldsymbol{\psi})-\sum_{i \in \mathcal{I}}\tilde{\boldsymbol{c}}_{t}[i]~{\Bigg|}
\hspace{-0.5mm} \sum_{\boldsymbol{\psi} \in \boldsymbol{\Psi}^{+}(\mathcal{I})} \hspace{-2mm} \left|q(\boldsymbol{\psi})-\hat{p}_{t}(\boldsymbol{\psi})\right|\leq C_{t}(\mathcal{I}) \hspace{-0.5mm} {\Bigg\}},
\end{align}
where $C_{t}(\mathcal{I}) = \min\left\{ 1, \sqrt{ \frac{2 \boldsymbol{\Psi}_{tot} \log{(2 T |\mathcal{P}(\mathcal{D})|/\delta)}}{N_{t}(\mathcal{I})}} \right\}$ and $\Delta_{|\boldsymbol{\Psi}^{+}(\mathcal{I})|}$ is a simplex in $|\boldsymbol{\Psi}^{+}(\mathcal{I})|$ dimensions. The optimization problem (\ref{eq:optimization-pi-q}) can be reduced to the following optimization problem 
% (See Section 1 of supplementary material for details.). %
(See Appendix \ref{app:Reduct} for details).
% a set of convex optimization problems.
\begin{align}
\label{eq:optimization}
\hat{V}_{t}(\mathcal{I})=\underset{q \in \Delta_{|\boldsymbol{\Psi}^{+}(\mathcal{I})|}}{\text{maximize}}~
&{\Bigg\{} \sum_{\boldsymbol{\psi} \in \boldsymbol{\Psi}^{+}(\mathcal{I})} q(\boldsymbol{\psi}) \tilde{r}_{t}^{\ast}(\boldsymbol{\psi}) - \sum_{i \in \mathcal{I}} \tilde{\boldsymbol{c}}_{t}[i] ~{\Bigg|}~ \sum_{\boldsymbol{\psi} \in \boldsymbol{\Psi}^{+}(\mathcal{I})} \left|q(\boldsymbol{\psi}) - \hat{p}_{t}(\boldsymbol{\psi})\right| \leq C_{t}(\mathcal{I}) {\Bigg\}},
\end{align}
where $\tilde{r}_{t}^{\ast}(\boldsymbol{\psi}) = \max_{a \in \mathcal{A}} \tilde{r}_{t}(a, \boldsymbol{\psi})$ is the optimistic reward estimate of the partial state vector $\boldsymbol{\psi}$ at time $t$. Problem (\ref{eq:optimization}) is solved by ranging the value of $q$ over the plausible candidate set of probabilities for $p(\boldsymbol{\psi})$. We denote the value of $q$ that solves (\ref{eq:optimization}) at time $t$ by $\tilde{p}_{t}(\boldsymbol{\psi})$. Note that, for each $\mathcal{I}$, the probability $\tilde{p}_{t}(\boldsymbol{\psi})$ denotes the optimistic probability estimate of the partial state vector $\boldsymbol{\psi} \in \boldsymbol{\Psi}^{+}(\mathcal{I})$ at time $t$. Moreover, $\hat{V}_{t}(\mathcal{I})$ represents the optimistic gain of a policy $\pi_{t}(\mathcal{I}) = (\mathcal{I}, \hat{h}_{t}(\boldsymbol{\psi}))$ that selects the observation set $\mathcal{I}$ and the action $\hat{h}_{t}(\boldsymbol{\psi})$ for any $\boldsymbol{\psi} \in \boldsymbol{\Psi}^{+}(\mathcal{I})$ at time $t$.
% observation probability estimate
% Note that since $\tilde{r}_{t}(a, \boldsymbol{\psi})$ is an optimistic (upper) estimate of $\bar{r}_{t}(a, \boldsymbol{\psi})$ (with high probability), and $\tilde{\boldsymbol{c}}_{t}[i]$ is a pessimistic (lower) estimate of $\bar{\boldsymbol{c}}_{t}[i]$ (with high probability), then $\hat{V}_{t}(\mathcal{I})$ is an optimistic estimate of $V_{t}(\mathcal{I})$. 
% We denote the result of this final optimization problem by $\hat{V}_{t}(\mathcal{I})$, which represents the optimistic gain of a policy $\pi_{t}(\mathcal{I}) = (\mathcal{I}, \hat{h}_{t}(\boldsymbol{\psi}))$ that selects the observation set $\mathcal{I}$ and the action $\hat{h}_{t}(\boldsymbol{\psi})$ for any $\boldsymbol{\psi} \in \boldsymbol{\Psi}^{+}(\mathcal{I})$ at time $t$.

% The algorithm then acts optimistically by choosing the observation set with the largest $\hat{V}_{t}(\mathcal{I})$ index. More specifically, 
At each time $t$, our algorithm solves (\ref{eq:optimization}) and acts optimistically by choosing the observation set and determining the action selection strategy as
\begin{align}
\label{eq:SelectedObservationSet}
    \hat{\mathcal{I}}_{t} = \argmaxA_{\mathcal{I} \in \mathcal{P}(\mathcal{D})}~\hat{V}_{t}(\mathcal{I}),
\end{align}
and
\begin{align}
\label{eq:DeterminedActionStrategy}
    \hat{h}_{t}(\boldsymbol{\psi}) = \argmaxA_{a \in \mathcal{A}}~\hat{r}_{t}(a, \boldsymbol{\psi}) + C_{t}(a, \boldsymbol{\psi}; w),
\end{align}
respectively. 
Afterward, NCC-UCRL2 pays the costs corresponding to the selected observation set $\hat{\mathcal{I}}_{t}$, observes the partial state vector $\boldsymbol{\psi}_{t} \in \boldsymbol{\Psi}^{+}(\hat{\mathcal{I}}_{t})$, and takes the action $a_{t} = \hat{h}_{t}(\boldsymbol{\psi}_{t})$. Finally, it receives the corresponding reward $r_{t}$ and updates the counters.
%-------------------------------------> Section Regret Analysis
\section{Theoretical Analysis}
\label{sec:regretanalysis}
In this section, we analyze the regret performance of NCC-UCRL2 algorithm in stationary and non-stationary environments. We first prove an upper bound on the expected regret of our algorithm by assuming that there is no change point in the environment. In the stationary case, we can choose $w = \Theta(T)$ to exploit the entire collected data when estimating the mean rewards and mean costs. In this case, as expected, NCC-UCRL2 achieves a sublinear regret.
%-------------------------------------> Theorem Regret Analysis
\begin{theorem} 
\label{thm:regret-stationary} 
If $\Upsilon_{T} = 0$, i.e., when the environment is stationary, with probability at least $1-3\delta$, the expected regret of NCC-UCRL2 is upper bounded as
\begin{align} \nonumber
\label{eq:final-regret-bound}
\mathcal{R}_{T}(\Pi) 
&\leq O {\Bigg(} T  {\Big(} \sqrt{ \frac{A \boldsymbol{\Psi}_{tot} \log{(T A \boldsymbol{\Psi}_{tot} w/\delta)}}{w} } + D \sqrt{ \frac{\log{(T D w/\delta)}}{w} } {\Big)} \\ \nonumber
&\hspace{20mm}+ \sqrt{T \log{(1/\delta)}} {\Big(} \sqrt{  A \boldsymbol{\Psi}_{tot} \log{(T A \boldsymbol{\Psi}_{tot} w/\delta)} } + D \sqrt{ \log{(T D w/\delta)} } {\Big)} \\
&\hspace{80mm}+ \sqrt{T |\mathcal{P}(D)| \boldsymbol{\Psi}_{tot} \log{(T |\mathcal{P}(\mathcal{D})|/\delta)}} {\Bigg)}.
\end{align}
Choosing $w = T$ results in
%\Theta(T)
\begin{align} \nonumber
\label{eq:final-regret-bound-2}
% \mathcal{R}_{T}(\Pi) \leq
% &~O {\Bigg(} \sqrt{ T A \boldsymbol{\Psi}_{tot} \log{(T A \boldsymbol{\Psi}_{tot} /\delta)} } + D \sqrt{ T \log{(T D /\delta)} } \\
% &\hspace{5mm}+ \sqrt{T \log{(1/\delta)}} {\Big(} \sqrt{  A \boldsymbol{\Psi}_{tot} \log{(T A \boldsymbol{\Psi}_{tot} /\delta)} } + D \sqrt{ \log{(T D /\delta)} } {\Big)} {\Bigg)}.
\mathcal{R}_{T}(\Pi) 
&\leq O {\Bigg(} {\Big(} 1 + \sqrt{\log{(1/\delta)}} {\Big)} {\Big(} \sqrt{ T A \boldsymbol{\Psi}_{tot} \log{(T A \boldsymbol{\Psi}_{tot} /\delta)} } + D \sqrt{ T \log{(T D /\delta)} } {\Big)} \\
&\hspace{80mm}+ \sqrt{T |\mathcal{P}(D)| \boldsymbol{\Psi}_{tot} \log{(T |\mathcal{P}(\mathcal{D})|/\delta)}} {\Bigg)}.
\end{align}
\end{theorem}
%-------------------------------------
%-------------------------------------> Proof
\begin{proof}
See Appendix \ref{app:TheoremOneProof}.
% See Section 4.1 of supplementary material.
\end{proof}
%-------------------------------------

% \textcolor{blue}{\textit{Theorem \ref{thm:regret-stationary}} states the agent's regret by using only recent $w$ rewards- and costs observations in stationary environments.}
The proof of Theorem \ref{thm:regret-stationary} is, to some extent, based on state-of-the-art techniques used in the literature to analyze regret bounds for optimistic bandit algorithms; nevertheless, some non-conventional parts appear in our derivation because we estimate the partial state probabilities using all observations, while the mean rewards and mean costs using the most recent ones in the window. Note that, in the optimization problem (\ref{eq:optimization-pi-q}), we use \textit{optimistic estimations} for rewards and partial state probabilities, whereas we rely on \textit{pessimistic ones} for costs by using the lower confidence bound on the mean costs in (\ref{eq:pessimisticcosts}). That results in several technical challenges in the theoretical analysis, for example, in Lemma \ref{lem:events}, where we bound the probability of failure 
% (See Section 4 of supplementary material). Lemma 3
(See Appendix \ref{app:mainresults}).
Moreover, proving the bound in (\ref{eq:alpha}) is challenging as the algorithm can choose more than one feature at a time. Hence, in (\ref{eq:alpha}), we consider the worst case of observing all the $D$ features' states at each time $t$.
% Please see the discussion below Equation (49) in Section 4.1 of supplementary material.
% Please s
% (49) (49)

In the next theorem, we establish an upper bound on the expected regret of NCC-UCRL2 in non-stationary environments. The regret analysis for non-stationary case is based on the theoretical analysis in Theorem \ref{thm:regret-stationary}.
%-------------------------------------> Theorem Regret Analysis
\begin{theorem} 
\label{thm:regret-nonstationary} 
If $\Upsilon_{T} > 0$, i.e., when the environment is non-stationary, with probability at least $1-3\delta$, the expected regret of NCC-UCRL2 is upper bounded as
\begin{align} \nonumber
\label{eq:final-regret-bound-nonstationary}
\mathcal{R}_{T}(\Pi) 
&\leq O {\Bigg(} w \Upsilon_{T} + T {\Big(} \sqrt{ \frac{A \boldsymbol{\Psi}_{tot} \log{(T A \boldsymbol{\Psi}_{tot} w/\delta)}}{w} } + D \sqrt{ \frac{\log{(T D w/\delta)}}{w} } {\Big)} \\ \nonumber
&\hspace{20mm}+ \sqrt{\Upsilon_{T} T  \log{(1/\delta)}} {\Big(} \sqrt{  A \boldsymbol{\Psi}_{tot} \log{(T A \boldsymbol{\Psi}_{tot} w/\delta)} } + D \sqrt{ \log{(T D w/\delta)} } {\Big)} \\
&\hspace{80mm}+ \sqrt{T |\mathcal{P}(D)| \boldsymbol{\Psi}_{tot} \log{(T |\mathcal{P}(\mathcal{D})|/\delta)}} {\Bigg)}.
\end{align}
Choosing $w = (T/\Upsilon_{T})^{2/3}$ results in
\begin{align} \nonumber
\label{eq:final-regret-bound-nonstationary-2}
% \mathcal{R}_{T}(\Pi) \leq
% &~O {\Bigg(} T^{2/3} \Upsilon_{T}^{1/3} {\Big(} \sqrt{ A \boldsymbol{\Psi}_{tot} \log{(T A \boldsymbol{\Psi}_{tot} /\delta)} } + D \sqrt{ \log{(T D /\delta)} } {\Big)} \\ \nonumber
% &\hspace{10mm}+ \sqrt{\Upsilon_{T} T  \log{(1/\delta)}} {\Big(} \sqrt{  A \boldsymbol{\Psi}_{tot} \log{(T A \boldsymbol{\Psi}_{tot} /\delta)} } + D \sqrt{ \log{(T D /\delta)} } {\Big)} \\
% &\hspace{70mm}+ \sqrt{T |\mathcal{P}(D)| \boldsymbol{\Psi}_{tot} \log{(T |\mathcal{P}(\mathcal{D})|/\delta)}} {\Bigg)}.
\mathcal{R}_{T}(\Pi) 
&\leq O {\Bigg(} {\Big(} T^{2/3} \Upsilon_{T}^{1/3} + \sqrt{\Upsilon_{T} T  \log{(1/\delta)}} {\Big)} {\Big(} \sqrt{ A \boldsymbol{\Psi}_{tot} \log{(T A \boldsymbol{\Psi}_{tot} /\delta)} } + D \sqrt{ \log{(T D /\delta)} } {\Big)} \\
&\hspace{80mm}+ \sqrt{T |\mathcal{P}(D)| \boldsymbol{\Psi}_{tot} \log{(T |\mathcal{P}(\mathcal{D})|/\delta)}} {\Bigg)}.
\end{align}
\end{theorem}
%-------------------------------------
%-------------------------------------> Proof
\begin{proof}
See Appendix \ref{app:TheoremTwoProof}.
% See Section 4.2 of supplementary material.
\end{proof}
%-------------------------------------
% The regret analysis for non-stationary case is based on the theoretical analysis in Theorem \ref{thm:regret-stationary}. During the stationary phases, the algorithm suffers the same sublinear regret shown in Theorem \ref{thm:regret-stationary}. When experiencing a change point, however, the algorithm suffers an extra $O(w)$ regret, while the second last term is scaled by a factor of $\sqrt{\Upsilon_{T}}$.

% (please see (20) and (22))
The analysis in Theorem \ref{thm:regret-nonstationary} is based on Theorem \ref{thm:regret-stationary}. During the stationary phases, the algorithm suffers the same sublinear regret proved in Theorem \ref{thm:regret-stationary}. When experiencing a change point, the algorithm suffers an extra $O(w)$ regret, while the second term in (\ref{eq:final-regret-bound}) scales by a factor of $\sqrt{\Upsilon_{T}}$. Our algorithm does not require the knowledge of $\Upsilon_{T}$ and guarantees a sublinear regret bound with a proper choice of $w$, as given by (\ref{eq:final-regret-bound-nonstationary-2}).
% The regret bound in (\ref{eq:final-regret-bound-nonstationary}) holds for any number of change points $0 < \Upsilon_{T}$.
% Even in that case, with a proper choice of $w$, NCC-UCRL2 achieves a sublinear regret bound in time, as given by (\ref{eq:final-regret-bound-nonstationary-2}).
% Numerical experiments validate the theory.
%-------------------------------------> Section: Numerical Analysis
\section{Numerical Analysis}
\label{sec:NumAnalysis}
In this section, via numerical experiments, we provide more insights into the effects of costly features on the performance of learning algorithms. Besides, we clarify how our proposed algorithm mitigates the adverse effects by observing only a subset of features' states. Moreover, we show that our algorithm efficiently adapts to environmental changes. We also compare the performance of our algorithm with conventional benchmarks using a real-world dataset. The source code for our algorithm and experiments in this paper are publicly available.\footnote{ \url{https://github.com/saeedghoorchian/NCC-Bandits.git}}
% The source code for our algorithm and experiments in this paper can be accessed via the anonymized link in the footnote.\footnote{Source code: \url{https://anonymous.4open.science/r/ncc_anonymized/}}
%are available publicly

% \subsubsection{Benchmark Policies}
\textbf{Benchmark Policies:}
We compare NCC-UCRL2 with the state-of-the-art contextual and context-agnostic algorithms. Contextual bandit algorithms in our experiment include \textbf{Sim-OOS} \cite{Atan21:DDO}, \textbf{PS-LinUCB} \cite{Xu20:CBB}, and \textbf{LinUCB} \cite{Li10:ACB}. Sim-OOS is designed for bandit problems with fixed costs for features' states observation in stationary environments. PS-LinUCB is designed for piece-wise stationary environments, but it is cost-agnostic. LinUCB is the final contextual bandit algorithm that is neither designed for changing environments nor costly features. In our experiment, similar to our algorithm, Sim-OOS can select any subset of features for state observation at each time of play. As a result, at each time, they pay the corresponding cost only for those selected features. PS-LinUCB and LinUCB always observe all features' states. Hence, they pay the full cost vector. We consider \textbf{UCB1} and \textbf{$\varepsilon$-Greedy} \cite{Auer02:FTA} as context-agnostic benchmarks as standard methods despite their weakness due to being blind to contextual information. We also consider a \textbf{random} policy that selects an action uniformly at random at each time. Context-agnostic algorithms do not incur any costs and only collect the rewards.
%

% \subsubsection{Nursery Dataset} 
\textbf{Nursery Dataset:}
We assess the performance of our algorithm on the Nursery dataset from the UCI Machine Learning Repository \cite{Dua17:UCI}. The dataset, derived from a hierarchical decision support system, includes applications for nursery schools and their target ranks that prioritize the applications and determine whether the child is recommended to be admitted to a nursery school. The applications are described using features that represent the socioeconomic status of the family. We consider $D = 5$ features: (i) Form of the family, (ii) number of children, (iii) financial standing of the family, (iv) housing conditions, and (v) health conditions of the applicant. In our experiment, we work with $A = 3$ target rank values ranging from $1$ to $3$ that indicate the given application is \textit{not recommended}, \textit{accepted with priority}, and \textit{accepted with special priority}, respectively. Taking an action is equivalent to recommending one particular rank for the given application. The agent receives reward $1$ if the correct rank is recommended, otherwise the reward is $0$. 

% \subsubsection{Experimental Setup}
\textbf{Experimental Setup:}
To simulate a piece-wise stationary reward generating process, we follow the approach proposed by \cite{Bouneffouf17:CAB}. At each change point, we shift all the target labels cyclically. This guarantees that the expected reward is piece-wise constant. In the context of decision support system for nursery school applications, such change points correspond to changes in preference of the decision-making authority over the applications.

We endow the features with random cost values. At each time $t$, the random cost of observation for each feature's state follows a normal distribution with a standard deviation of $0.001$ and a piece-wise constant mean. We select the mean values of cost distributions uniformly at random from the interval $[0.03; 0.08]$. Therefore, the total observation cost of a full state vector at each time amounts to $15-40\%$ of the maximum reward. The range of costs are chosen based on two factors: (i) It should be high enough to prevent the algorithm from observing all features' states at all times and, (ii) low enough to incentivize the algorithm considerably to pay for state observation in order to find the optimal observations. In the nursery application ranking scenario, the state observation costs can be thought of as the efforts required to acquire the information about the applicant. Such efforts may include the time or other related expenses spent to obtain the information.

% perform cross-validation in order
% Therefore, we use roughly $20\%$ of the total data as the tuning set.
We split the data into train and validation (tuning) sets in approximately 80:20 ratio with $10000$ and $2630$ data samples, respectively. More specifically, we sample $2630$ data points at random and use them to tune the parameters of algorithms. The parameters of those benchmark algorithms that are originally designed for stationary environments are tuned without introducing non-stationarity in the validation set. To tune the parameters of NCC-UCRL2 and PS-LinUCB, we consider $2$ change points in mean rewards, but no change points in mean costs. For more details on the tuning process of the parameters, please see Appendix \ref{app:AddInfo}.
% Section 5 of the supplementary material.

We run the experiment for $T = 10000$ time steps by revealing applications to the algorithms one at a time. We consider a maximum of $\Upsilon_{T} = 7$ change points in our experiment, with change points in the mean rewards and the mean costs at times $\{1000, 2000, 5000, 8000\}$ and $\{3000, 5000, 7000, 9000\}$, respectively. Note that the change points are not necessarily identical; the mean rewards and mean costs do not always change simultaneously at a change point. In Appendix \ref{app:AddInfo},
% In Section 5 of the supplementary material, 
we elaborate more on the settings of mean rewards and mean costs. \textbf{Table \ref{table:PolicyParams}} 
% 2 in the supplementary material
lists the tuned parameters of algorithms used in our simulation. For NCC-UCRL2, we set $\delta = 0.04$ and choose the window parameter $w = 250$.
% In Appendix \ref{app:AddInfo}, we elaborate more on the settings of mean rewards and mean costs. \textbf{Table \ref{table:PolicyParams}} lists the tuned parameters of algorithms used in our simulation.

% \subsubsection{Regret Comparison}
\textbf{Regret Comparison:}
We run the algorithms using the aforementioned setup. \textbf{Fig. \ref{fig:Regret}} depicts the trend of cumulative regret over time for each policy. We average the results over $5$ independent runs. Here, the instantaneous regret at each time is defined based on the instantaneous gain, which is the obtained reward minus the total paid observation costs at every round. As we see, NCC-UCRL2 detects the changes in the mean rewards or mean costs faster than all other policies and therefore has a superior performance. Besides, as NCC-UCRL2 uses only the last $w$ observations to estimate the mean rewards and mean costs, it has a smooth curve around change points. These advantages are despite the fact that NCC-UCRL2 only observes a subset of features' states at each time.
%------------------------------------->
\begin{figure}[t!]
    \centering
    \includegraphics[width=0.65\textwidth]{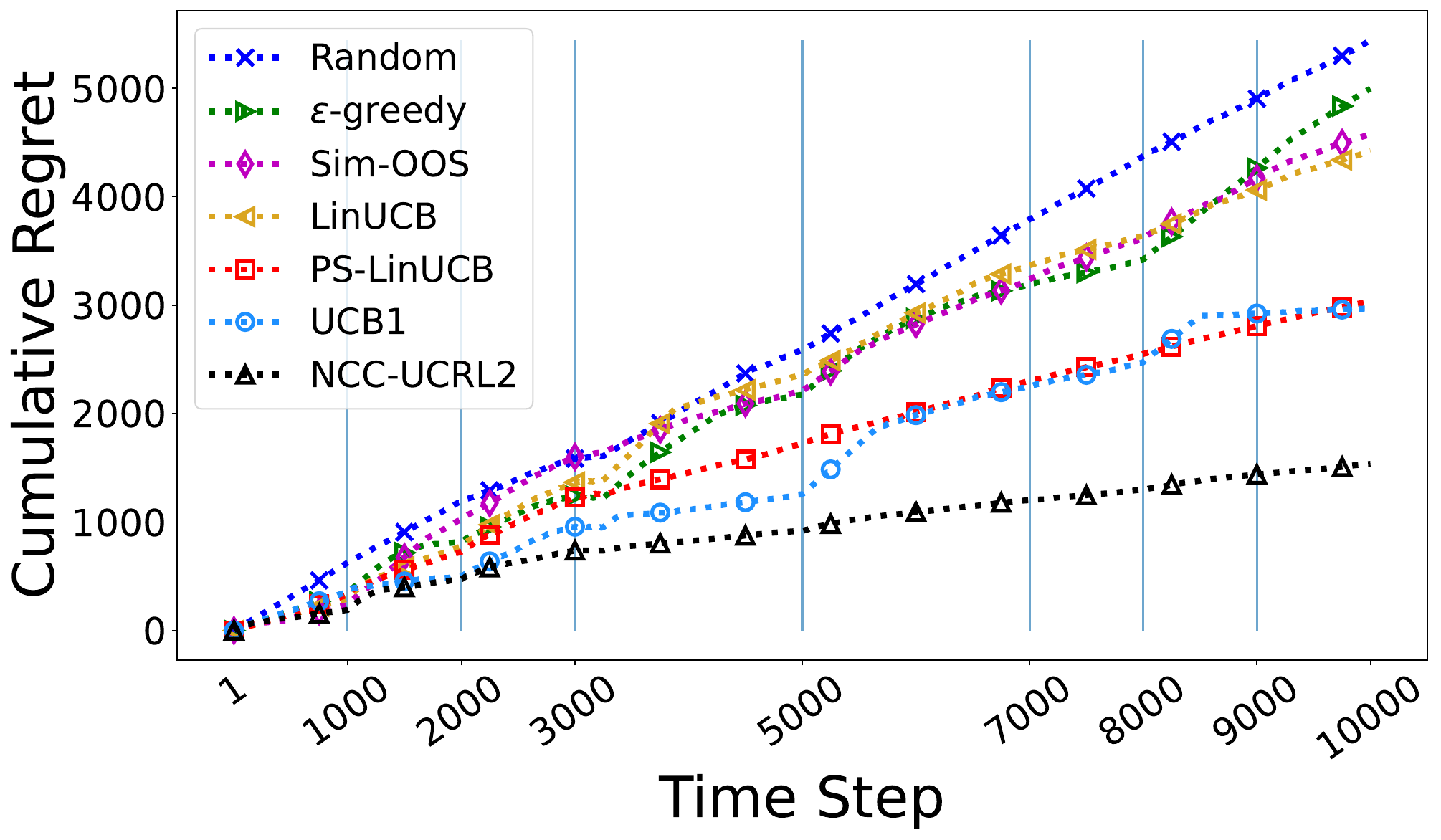}
    \caption{Cumulative regret of different policies. Vertical lines show the change points.}
    % dotted
    \label{fig:Regret}
\end{figure}
%-------------------------------------

%------------------------------------->
\begin{figure}[b!]
    \centering
    \includegraphics[width=0.7\textwidth]{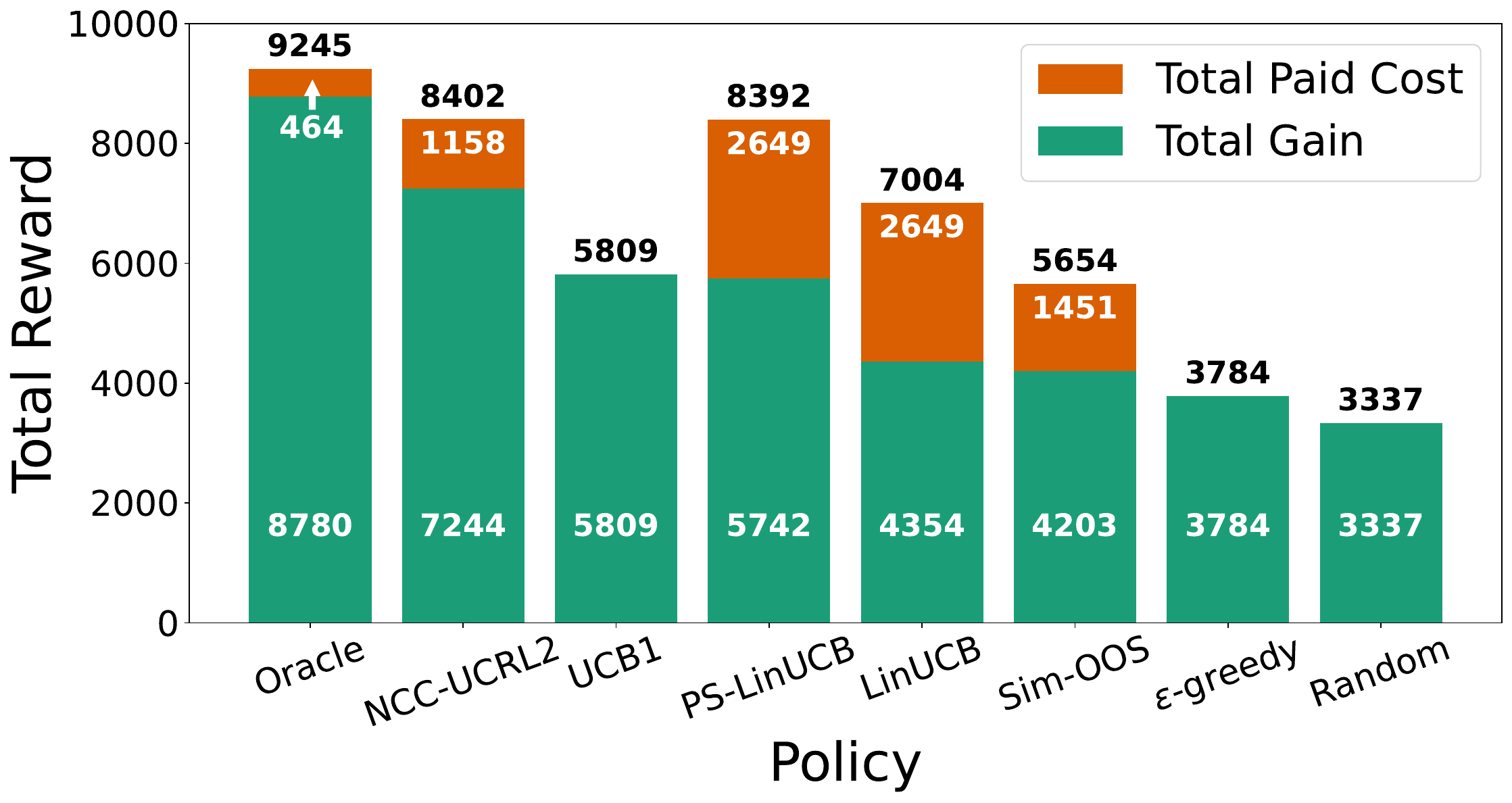}
    \caption{Total reward (number on top of bar), gain (number in green), and cost (number in brown) for each policy. Values are rounded to the nearest integers.}
    \label{fig:RewardBarChart}
\end{figure}
%-------------------------------------
% \subsubsection{Gain Comparison}
\textbf{Gain Comparison:}
In \textbf{Fig. \ref{fig:RewardBarChart}}, we show the policies' total reward, gain, and cost. It also compares them with the oracle. In this figure, the height of each bar shows the total accumulated reward of each policy which is equal to the total gain (green part) plus the total cost (brown part). NCC-UCRL2 accumulates the highest rewards during the experiment among the benchmark policies. The accumulated reward of PS-LinUCB is almost the same as that of our algorithm; it receives only about $0.1\%$ less reward than NCC-UCRL2. However, the total gain of PS-LinUCB is $20\%$ lower due to higher paid costs as it observes all the features' states at all times. On the contrary, NCC-UCRL2 adaptively learns the optimal state observations while it observes only a fraction of features' states at each time. As a result, NCC-UCRL2 incurs less cost, hence a higher performance concerning the accumulated gain. The two counterparts of NCC-UCRL2 and PS-LinUCB that suit stationary environments, i.e., Sim-OOS and LinUCB, exhibit a similar pattern for the total costs; nevertheless, Sim-OOS achieves lower accumulated reward compared to LinUCB, which shows the importance of learning the optimal observations in a non-stationary environment. Note that Sim-OOS fails in our experiment as it does not consider the pessimistic selection of random costs and cannot adapt to drifts.

%------------------------------------->
\begin{figure}[t!]
    \centering
    \includegraphics[width=0.7\textwidth]{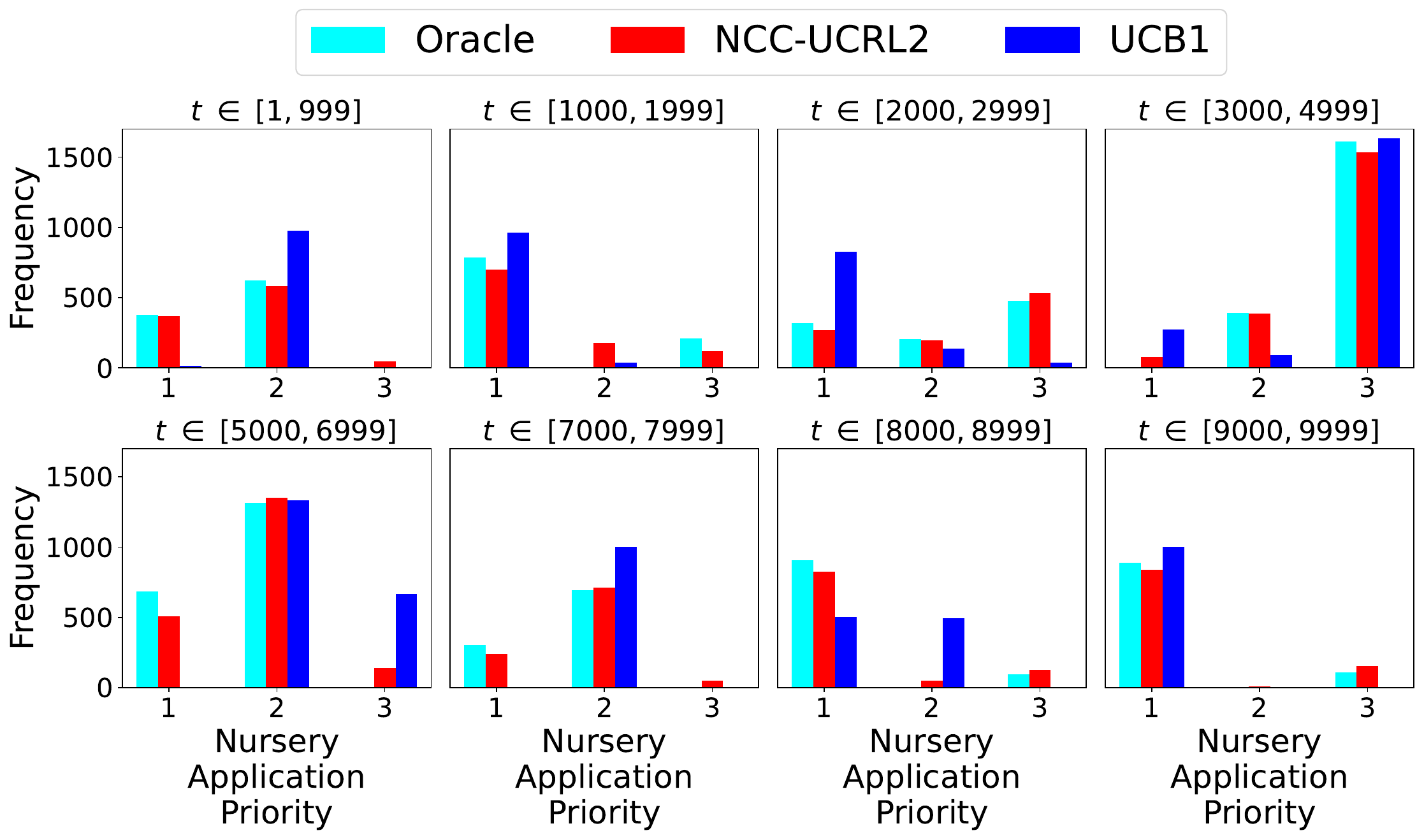}
    \caption{Comparison of priority recommendations of the oracle, NCC-UCRL2, and UCB1 in each stationary period.}
    \label{fig:ArmsHistogram}
\end{figure}
%-------------------------------------
% \subsubsection{Adaptation to the Preference Volatility} 
\textbf{Adaptation to the Preference Volatility:} 
In \textbf{Fig. \ref{fig:ArmsHistogram}}, we plot the histograms of nursery application priorities recommended by the oracle, NCC-UCRL2, and UCB1 for each of the stationary periods. Our algorithm closely follows the arm choice pattern of the oracle, which means that it can quickly adapt to changes in preference over applications. On the other hand, UCB1 cannot always adapt to sudden changes in the environment. We particularly consider UCB1 in this analysis to show the following: Although UCB1 achieves the second highest gain amongst the benchmarks, it fails to provide tailored recommendations when the environment parameters undergo abrupt changes.

We perform further numerical analysis on the performance of NCC-UCRL2 algorithm and present the results in Appendix \ref{app:AddExp}.
% Section 6 of supplementary material.

%-------------------------------------> Section: Conclusion
\section{Conclusion}
\label{sec:Conclusion}
We introduced the NCC bandit framework, where information acquisition is costly and the environment is non-stationary. We developed a decision-making policy, namely NCC-UCRL2, that mitigates the effects of costs by observing only a subset of features. We proved that NCC-UCRL2 achieves a sublinear regret bound in time. Our proposed framework is applicable in several contexts, such as online advertising problems, medical treatment recommendations, edge computing, and stock trading. We applied our method to recommend priority ranks for nursery school applications. The experiments showed that NCC-UCRL2 outperforms several state-of-the-art bandit algorithms.
% The first future research direction would be to extend the current framework by considering delayed feedback. Another potential extension of our work would be to consider sequential state observations, where at each round of decision-making, the agent selects features sequentially to observe their states and improves the future feature selections based on the already selected observations.
We study the general NCC bandit problem, where the reward can take any form, linear or nonlinear. Besides, the number of state observations can be arbitrarily large. A potential future research direction would be to allow for restrictive assumptions on the number of state observations or the space of reward functions. In such cases, the dependence of the regret bound on the number of features and partial states diminishes.
\section{Appendix}
\label{sec:App}
%
% \appendices
%
%-------------------------------------> Reduction
\subsection{Reduction of Optimization Problem (\ref{eq:optimization-pi-q})}
% EDUCTION OF OPTIMIZATION PROBLEM (16)} %(\ref{eq:optimization-pi-q})
\label{app:Reduct}
We can solve the optimization problem (\ref{eq:optimization-pi-q}) by first fixing the observation set $\mathcal{I}$ and the probabilities $q$, and then, maximizing only with respect to the action selection function $h$. For a fixed $\mathcal{I}$ and $q$, let $\hat{h}_{t}^{\mathcal{I},q}(\boldsymbol{\psi})$ denote the action function that maximizes the optimization problem (16). We have $\hat{h}_{t}^{\mathcal{I},q}(\boldsymbol{\psi}) = \hat{h}_{t}(\boldsymbol{\psi}) = \arg\max_{a \in \mathcal{A}} \tilde{r}_{t}(a, \boldsymbol{\psi})$. Therefore, By fixing $h$ to $\hat{h}_{t}^{\mathcal{I},q}$ in (\ref{eq:optimization-pi-q}), we obtain the following optimization problem.
\begin{align} 
\label{eq:optimization-I-q}
\max_{\mathcal{I}, q \in \Delta_{|\boldsymbol{\Psi}^{+}(\mathcal{I})|}} &{\Bigg\{} \sum_{\boldsymbol{\psi} \in \boldsymbol{\Psi}^{+}(\mathcal{I})} q(\boldsymbol{\psi}) \tilde{r}_{t}^{\ast}(\boldsymbol{\psi}) 
- \sum_{i \in \mathcal{I}} \tilde{\boldsymbol{c}}_{t}[i] ~{\Bigg|}~ \sum_{\boldsymbol{\psi} \in \boldsymbol{\Psi}^{+}(\mathcal{I})} \left|q(\boldsymbol{\psi}) - \hat{p}_{t}(\boldsymbol{\psi})\right| \leq C_{t}(\mathcal{I}) {\Bigg\}}.
\end{align}
We solve the problem (\ref{eq:optimization-I-q}) by first fixing the observation set $\mathcal{I}$ and then, maximizing with respect to the probabilities $q$. This results in the optimization problem (\ref{eq:optimization}). %(\ref{eq:optimization}) (17)

%-------------------------------------> Notations
% \subsection{Notations}
\subsection{Notations}
\label{app:notation}
Before proceeding to the proof, in the following we introduce some important notations together with their definitions.
% To avoid any ambiguity, in this document we continue the (equation-) numbering of the main text.

% is the expected reward of action $a$ and partial state vector $\boldsymbol{\psi}$ minus the observation cost of $\boldsymbol{\psi}$. Formally,
We define the expected gain of an action $a$ and a partial state vector $\boldsymbol{\psi}$ as $g_{t}(a, \boldsymbol{\psi}) = \bar{r}_{t}(a, \boldsymbol{\psi}) - \sum_{i \in \mathscr{D}(\boldsymbol{\psi})} \bar{\boldsymbol{c}}_{t}[i]$. In addition, we define $\tilde{g}_{t}(a, \boldsymbol{\psi}) = \tilde{r}_{t}(a, \boldsymbol{\psi}) - \sum_{i \in \mathscr{D}(\boldsymbol{\psi})} \tilde{\boldsymbol{c}}_{t}[i]$. For ease of presentation, we introduce new vector notations. We collect the probability distributions for partial state vectors $\boldsymbol{\psi} \in \boldsymbol{\Psi}^{+}(\hat{\mathcal{I}}_{t})$ in a vector and denote it by $\mathbf{P}(\hat{\mathcal{I}}_{t}) = [p(\boldsymbol{\psi})]_{\boldsymbol{\psi} \in \boldsymbol{\Psi}^{+}(\hat{\mathcal{I}}_{t})}$. Similarly, we define $\tilde{\mathbf{P}}_{t}(\hat{\mathcal{I}}_{t}) = [\tilde{p}_{t}(\boldsymbol{\psi})]_{\boldsymbol{\psi} \in \boldsymbol{\Psi}^{+}(\hat{\mathcal{I}}_{t})}$, $\hat{\mathbf{P}}_{t}(\hat{\mathcal{I}}_{t}) = [\hat{p}_{t}(\boldsymbol{\psi})]_{\boldsymbol{\psi} \in \boldsymbol{\Psi}^{+}(\hat{\mathcal{I}}_{t})}$, $\mathbf{G}_{t}(\hat{\mathcal{I}}_{t}) = [g_{t}(\hat{h}_{t}(\boldsymbol{\psi}), \boldsymbol{\psi})]_{\boldsymbol{\psi} \in \boldsymbol{\Psi}^{+}(\hat{\mathcal{I}}_{t})}$, $\tilde{\mathbf{G}}_{t}(\hat{\mathcal{I}}_{t}) = [\tilde{g}_{t}(\hat{h}_{t}(\boldsymbol{\psi}), \boldsymbol{\psi})]_{\boldsymbol{\psi} \in \boldsymbol{\Psi}^{+}(\hat{\mathcal{I}}_{t})}$. Moreover, we define $n_{t}(\mathcal{I}) = \sum_{\tau = 1}^{t} \mathbbm{1}\{\mathcal{I}_{\tau} = \mathcal{I}\}$.

% which is the solution of (\ref{eq:optimization-pi-q})
Let $\tilde{\rho}_{t}$ denote the optimistic gain at time $t$. Based on the aforementioned definitions, we have $\tilde{\rho}_{t} = \langle \tilde{\boldsymbol{P}}_{t}(\hat{\mathcal{I}}_{t}), \tilde{\boldsymbol{G}}_{t}(\hat{\mathcal{I}}_{t}) \rangle$, where $\langle \cdot, \cdot \rangle$ denotes the dot product between two vectors. Therefore,
\begin{align} \nonumber
    \tilde{\rho}_{t} 
    &= \langle \tilde{\boldsymbol{P}}_{t}(\hat{\mathcal{I}}_{t}), \tilde{\boldsymbol{G}}_{t}(\hat{\mathcal{I}}_{t}) \rangle  \\ \nonumber
    &= \sum_{\boldsymbol{\psi} \in \boldsymbol{\Psi}^{+}(\hat{\mathcal{I}}_{t})} \tilde{p}_{t}(\boldsymbol{\psi}) \tilde{g}_{t}(\hat{h}_{t}(\boldsymbol{\psi}), \boldsymbol{\psi}) \\
    &= \sum_{\boldsymbol{\psi} \in \boldsymbol{\Psi}^{+}(\hat{\mathcal{I}}_{t})} \tilde{p}_{t}(\boldsymbol{\psi}) {\Big[} \hat{r}_{t}(\hat{h}_{t}(\boldsymbol{\psi}), \boldsymbol{\psi}) + C_{t}(\hat{h}_{t}(\boldsymbol{\psi}), \boldsymbol{\psi}; w) - \sum_{i \in \mathscr{D}(\boldsymbol{\psi})} [\hat{\boldsymbol{c}}_{t}[i] - C_{t}(i; w)] {\Big]}.
\end{align}
At each time $t$, we use $\rho_{t}$ to denote the expected gain of the agent that follows our proposed policy. Let $[T] = \{1, 2, \dots, T\}$. We define the following events which we use in the subsequent proofs.
\begin{align} \label{eq:event1}
    \mathcal{E}_{1} &= \mathbbm{1}{\Big\{}\exists t \in [T], ~\textup{s.t.}~ \rho_{t} \leq \tilde{\rho}_{t} {\Big\}}, \\ \label{eq:event2}
    \mathcal{E}_{2} &= \mathbbm{1}{\Big\{}\exists t \in [T], \exists \mathcal{I} \in \mathcal{P}(\mathcal{D}), ~\textup{s.t.}~ \norm{\hat{\mathbf{P}}_{t}(\mathcal{I}) - \mathbf{P}(\mathcal{I})}_{1} \leq C_{t}(\mathcal{I}) {\Big\}}, \\ \label{eq:event3}
    \mathcal{E}_{3} &= \mathbbm{1}{\Big\{}\exists t \in [T], \exists a \in \mathcal{A}, \exists \boldsymbol{\psi} \in \boldsymbol{\Psi}, ~\textup{s.t.}~ |\hat{r}_{t}(a,\boldsymbol{\psi}) - \bar{r}_{t}(a, \boldsymbol{\psi})| \leq C_{t}(a, \boldsymbol{\psi}; w) {\Big\}}, \\ \label{eq:event4}
    \mathcal{E}_{4} &= \mathbbm{1}{\Big\{}\exists t \in [T], \exists i \in \mathcal{D}, ~\textup{s.t.}~ |\hat{\boldsymbol{c}}_{t}[i] - \bar{\boldsymbol{c}}_{t}[i]| \leq C_{t}(i; w) {\Big\}}.
\end{align}
Finally, by $\bar{\mathcal{E}}$, we denote the complement of an event $\mathcal{E}$.

%-------------------------------------> Auxiliary
% \subsection{Auxiliary Results}
\subsection{Auxiliary Results}
\label{app:aux}
%
%-------------------------------------> Azuma-Hoeffding inequality
%
\begin{lemma}{\textup{\cite{Azuma67:WSC}}}
\label{lem:Hoeffding}
Let $x_{1}, x_{2}, \dots, x_{n}$ be random variables and $x_{i} \in [0,b_{i}]$, $\forall i$. Moreover, $\mathbb{E}[x_{i} | x_{1}, \dots, x_{i-1}] = \beta$, for all $i = 1, \dots, n$. Then, for all $B \geq 0$,
\begin{equation}
     \mathbb{P} {\Bigg[} {\Big |} \sum_{i = 1}^{n} x_{i} - n \beta {\Big |} \geq B {\Bigg]} \leq e^{- \frac{2 B^{2}}{\sum_{i = 1}^{n} b_{i}^{2}} }.
\end{equation}
\end{lemma}
%
%-------------------------------------
%-------------------------------------> Weissman Lemma
\begin{lemma}{\textup{\cite{Weissman03:IFT}}}
\label{lemma:Weissman}  
Let $\mathcal{Z} = \{1, 2, \ldots, z\}$ and assume $\boldsymbol{P}$ represents a probability distribution on $\mathcal{Z}$. Moreover, consider $X^{n} = X_{1}, \ldots, X_{n} \in \mathcal{Z}$ to be i.i.d. random variables that are distributed according to $\boldsymbol{P}$. Let $\hat{\boldsymbol{P}}$ be the empirical estimate of $\boldsymbol{P}$, that is defined for each $z \in \mathcal{Z}$ as $\hat{\boldsymbol{P}}(z) = \frac{1}{n} \sum_{t = 1}^{n} \mathbbm{1}\{X_{t} = z\}$. Then, for any $\delta > 0$,
\begin{align}
\mathbb{P}\left[\|\boldsymbol{P} - \hat{\boldsymbol{P}}\|_{1} \geq \sqrt{ \frac{2 Z \log{\frac{2}{\delta}}}{n} } \right] \leq \delta, 
\end{align}
where $\|\boldsymbol{P} - \hat{\boldsymbol{P}}\|_{1} = \sum_{z = 1}^{Z} |\boldsymbol{P}(z) - \hat{\boldsymbol{P}}(z)|$ is the $L_1$ norm.
\end{lemma}
%-------------------------------------
%-------------------------------------> Main Results
% \subsection{Main Results}
\subsection{Main Results}
\label{app:mainresults}
Before we present the proof of Theorem 1 and 2, we need to prove the following lemma that shows the events defined in (\ref{eq:event1})-(\ref{eq:event4}) fail with a low probability. %Theorem \ref{thm:regret-stationary} and \ref{thm:regret-nonstationary}
%-------------------------------------> Lemma Prob of Event Fails
% {\textbf{Bounding the probability of the failure event}\\}
\begin{lemma}
\label{lem:events}
Consider the events defined in (\ref{eq:event1})-(\ref{eq:event4}). Then,
\begin{align}
\label{lem:prob_failure_events}
    \mathbb{P}[\bar{\mathcal{E}}_{1} \cup \bar{\mathcal{E}}_{2} \cup \bar{\mathcal{E}}_{3} \cup \bar{\mathcal{E}}_{4}] \leq \mathbb{P}[\bar{\mathcal{E}}_{2} \cup \bar{\mathcal{E}}_{3} \cup \bar{\mathcal{E}}_{4}] \leq 3 \delta.
\end{align}
\end{lemma} 
%
% $\bar{\mathcal{E}}^{'}~\&~\bar{\mathcal{E}}^{''} \Rightarrow \bar{\mathcal{E}}$
% $\mathcal{E}_{1}~\&~\mathcal{E}_{2} \Rightarrow \mathcal{E}$
\begin{proof}
First, note that if $\mathcal{E}_{2}$, $\mathcal{E}_{3}$, and $\mathcal{E}_{4}$ hold, the following is true: (i) $p(\boldsymbol{\psi})$ belongs to the set of distributions over which the solution of (17) is computed,  (ii) $\bar{r}_{t}(a, \boldsymbol{\psi}) \leq \hat{r}_{t}(a,\boldsymbol{\psi}) + C_{t}(a, \boldsymbol{\psi}; w)$, and (iii) $\hat{\boldsymbol{c}}_{t}[i] - C_{t}(i; w) \leq \bar{\boldsymbol{c}}_{t}[i]$. Therefore, %(\ref{eq:optimization})
\begin{align} \nonumber
    \tilde{\rho}_{t} 
    &=  \langle \tilde{\mathbf{P}}_{t}(\hat{\mathcal{I}}_{t}), \tilde{\mathbf{G}}_{t}(\hat{\mathcal{I}}_{t}) \rangle \\ \nonumber
    &= \sum_{\boldsymbol{\psi} \in \boldsymbol{\Psi}^{+}(\hat{\mathcal{I}}_{t})} \tilde{p}_{t}(\boldsymbol{\psi}) {\Bigg[} \hat{r}_{t}(\hat{h}_{t}(\boldsymbol{\psi}), \boldsymbol{\psi}) + C_{t}(\hat{h}_{t}(\boldsymbol{\psi}), \boldsymbol{\psi}; w) - \sum_{i \in \mathscr{D}(\boldsymbol{\psi})} \left[ \hat{\boldsymbol{c}}_{t}[i] - C_{t}(i; w) \right] {\Bigg]} \\ \nonumber
    &\geq \sum_{\boldsymbol{\psi} \in \boldsymbol{\Psi}^{+}(\hat{\mathcal{I}}_{t})} p(\boldsymbol{\psi}) \left[ \bar{r}_{t}(\hat{h}_{t}(\boldsymbol{\psi}), \boldsymbol{\psi}) - \sum_{i \in \mathscr{D}(\boldsymbol{\psi})} \bar{\boldsymbol{c}}_{t}[i] \right] \\
    &= \sum_{\boldsymbol{\psi} \in \boldsymbol{\Psi}^{+}(\hat{\mathcal{I}}_{t})} p(\boldsymbol{\psi}) \bar{r}_{t}(\hat{h}_{t}(\boldsymbol{\psi}), \boldsymbol{\psi})  - \sum_{i \in \mathscr{D}(\boldsymbol{\psi})} \bar{\boldsymbol{c}}_{t}[i] = \rho_{t}, 
\end{align}
which implies that $\mathcal{E}_{1}$ is also true. This proves the first inequality in (\ref{lem:prob_failure_events}).

% we take the union bound over the events that $L_{1}$ error of $\hat{\mathbf{P}}_{t}$ in estimating $\mathbf{P}$ is larger than the confidence bound $C_{T}(\mathcal{I})$
Second, we bound each individual failure event in the following. For $\bar{\mathcal{E}}_{2}$, by taking the union bound and using the concentration bound stated in Lemma \ref{lemma:Weissman}, we obtain
\begin{align}
    \mathbb{P}[\bar{\mathcal{E}}_{2}] 
    \leq
    \sum_{t = 1}^{T} \sum_{\mathcal{I} \in \mathcal{P}(\mathcal{D})}^{} \mathbb{P}\left[ \norm{\mathbf{P}(\mathcal{I}) - \hat{\mathbf{P}}_{t}(\mathcal{I})}_{1} \geq C_{t}(\mathcal{I}) \right] \leq \delta.
\end{align}
For $\bar{\mathcal{E}}_{3}$ and $\bar{\mathcal{E}}_{4}$, similar to \cite{Vernade20:ORC}, we use the Hoeffding-Azuma inequality stated in Lemma \ref{lem:Hoeffding}. More precisely, let $\hat{r}_{t, u}(a,\boldsymbol{\psi})$ denote the empirical estimate of $\bar{r}_{t}(a,\boldsymbol{\psi})$ using the first $u$ reward observations corresponding to the action $a$ and the partial state vector $\boldsymbol{\psi}$ in the window $[t-w, t-1]$. Similarly, let $\hat{\boldsymbol{c}}_{t, u}[i]$ denote the the empirical estimate of $\bar{\boldsymbol{c}}_{t}[i]$ using the first $u$ cost observations corresponding to the feature $i \in \mathcal{D}$ in the window $[t-w, t-1]$. We have $\hat{r}_{t, N_{t}(a,\boldsymbol{\psi}; w)}(a,\boldsymbol{\psi}) = \hat{r}_{t}(a,\boldsymbol{\psi})$ and $\hat{\boldsymbol{c}}_{t, N_{t}(i; w)}[i] = \hat{\boldsymbol{c}}_{t}[i]$. Then,
\begin{align} 
    &\mathbb{P}[\bar{\mathcal{E}}_{3}] 
    \leq
    \sum_{t = 1}^{T} \sum_{a \in \mathcal{A}}^{} \sum_{\boldsymbol{\psi} \in \boldsymbol{\Psi}}^{} \sum_{u = 1}^{w} \mathbb{P} {\Bigg[} |\bar{r}_{t}(a,\boldsymbol{\psi}) - \hat{r}_{t, u}(a, \boldsymbol{\psi})| 
    \geq \sqrt{\frac{\log{(T A \boldsymbol{\Psi}_{tot} w/\delta)}}{u}} {\Bigg]} \leq \delta, \\ 
    &\mathbb{P}[\bar{\mathcal{E}}_{4}] 
    \leq
    \sum_{t = 1}^{T} \sum_{i \in \mathcal{D}}^{} \sum_{u = 1}^{w} \mathbb{P} {\Bigg[} |\bar{\boldsymbol{c}}_{t}[i] - \hat{\boldsymbol{c}}_{t, u}[i]| \geq\sqrt{\frac{2 \log{(T D w/\delta)}}{u}} {\Bigg]} \leq \delta.
\end{align}
Therefore, we prove the second inequality in (\ref{lem:prob_failure_events}) and conclude the proof.
\end{proof}
%-------------------------------------

%-------------------------------------> Proof of Theorem 1
% \ref{app:TheoremOneProof}
% \subsubsection{Proof of Theorem \ref{thm:regret-stationary}}
\subsubsection{Proof of Theorem \ref{thm:regret-stationary}} %\ref{thm:regret-stationary} 1
\label{app:TheoremOneProof}
%
%------------------------------------->
\begin{proof}
Assume that the events $\mathcal{E}_{1}$, $\mathcal{E}_{2}$, $\mathcal{E}_{3}$, and $\mathcal{E}_{4}$, defined in (\ref{eq:event1})-(\ref{eq:event4}), hold. Note that, based on the definition of optimal policy in (\ref{eq:optimalpolicy}), when $\mathcal{E}_{1}$ happens, we have $\tilde{\rho}_{t} \geq \rho_{t}^{\ast}$. Then, we observe that %(\ref{eq:optimalpolicy}) (4)
\begin{align}
\label{eq:decomposition} \nonumber
    \mathcal{R}_{T}(\Pi) 
    &= \sum_{t = 1}^{T} \left[\rho_{t}^{\ast} - \rho_{t}\right]
    \leq \sum_{t = 1}^{T} \left[\tilde{\rho}_{t} - \rho_{t}\right] \\ \nonumber
    &= \sum_{t = 1}^{T} \sum_{\boldsymbol{\psi} \in \boldsymbol{\Psi}^{+}(\hat{\mathcal{I}}_{t})} \left[\tilde{p}_{t}(\boldsymbol{\psi}) \tilde{g}_{t}(\hat{h}_{t}(\boldsymbol{\psi}), \boldsymbol{\psi}) - p(\boldsymbol{\psi}) g_{t}(\hat{h}_{t}(\boldsymbol{\psi}), \boldsymbol{\psi})\right] \\ \nonumber
    & = \underbrace{\sum_{t = 1}^{T} \sum_{\boldsymbol{\psi} \in \boldsymbol{\Psi}^{+}(\hat{\mathcal{I}}_{t})} \left[\tilde{p}_{t}(\boldsymbol{\psi}) - p(\boldsymbol{\psi})\right] \tilde{g}_{t}(\hat{h}_{t}(\boldsymbol{\psi}), \boldsymbol{\psi})}_{\Delta_{1}} \\ 
    &\hspace{5mm}+ \underbrace{\sum_{t = 1}^{T} \sum_{\boldsymbol{\psi} \in \boldsymbol{\Psi}^{+}(\hat{\mathcal{I}}_{t})} p(\boldsymbol{\psi}) \left[\tilde{g}_{t}(\hat{h}_{t}(\boldsymbol{\psi}), \boldsymbol{\psi}) - g_{t}(\hat{h}_{t}(\boldsymbol{\psi}), \boldsymbol{\psi})\right]}_{\Delta_{2}}.
\end{align}
We bound each term individually. For $\Delta_{1}$, we have
\begin{align}
\label{eq:first-delta-1} \nonumber
    \sum_{t = 1}^{T} \sum_{\boldsymbol{\psi} \in \boldsymbol{\Psi}^{+}(\hat{\mathcal{I}}_{t})} \left[\tilde{p}_{t}(\boldsymbol{\psi}) - p(\boldsymbol{\psi})\right] \tilde{g}_{t}(\hat{h}_{t}&(\boldsymbol{\psi}), \boldsymbol{\psi}) 
    = \sum_{t = 1}^{T} \langle \tilde{\mathbf{P}}_{t}(\hat{\mathcal{I}}_{t}) - \mathbf{P}(\hat{\mathcal{I}}_{t}), \tilde{\mathbf{G}}_{t}(\hat{\mathcal{I}}_{t}) \rangle \\ \nonumber
    &\stackrel{(a)}{\leq} \sum_{t = 1}^{T} \norm{\tilde{\mathbf{P}}_{t}(\hat{\mathcal{I}}_{t}) - \mathbf{P}(\hat{\mathcal{I}}_{t})}_{1} \norm{\tilde{\mathbf{G}}_{t}(\hat{\mathcal{I}}_{t})}_{\infty} \\ \nonumber
    &\stackrel{(b)}{\leq} \sum_{t = 1}^{T} 2 C_{t}(\hat{\mathcal{I}}_{t}) \\
    &= 2 \sqrt{ \boldsymbol{\Psi}_{tot} \log{(2 T |\mathcal{P}(\mathcal{D})|/\delta)} } \sum_{t = 1}^{T} \frac{1}{\sqrt{N_{t}(\hat{\mathcal{I}}_{t})}},
\end{align}
where $(a)$ follows from Cauchy-Schwarz inequality and $(b)$ holds since event $\mathcal{E}_{2}$ occurs and $\norm{\tilde{\mathbf{G}}_{t}(\hat{\mathcal{I}}_{t})}_{\infty} \leq 2$. To bound the sum in the last term of (\ref{eq:first-delta-1}), we write
\begin{align}
\label{eq:I-counter-bound} \nonumber
    \sum_{t = 1}^{T} \frac{1}{\sqrt{N_{t}(\hat{\mathcal{I}}_{t})}}
     = \sum_{t = 1}^{T} \sum_{\mathcal{I} \in \mathcal{P}(D)} \frac{ \mathbbm{1}\{\hat{\mathcal{I}}_{t} = \mathcal{I}\} }{ \sqrt{N_{t}(\mathcal{I})} } 
     = \sum_{\mathcal{I} \in \mathcal{P}(D)} \sum_{t = 1}^{T}  \frac{ \mathbbm{1}\{\hat{\mathcal{I}}_{t} = \mathcal{I}\} }{ \sqrt{N_{t}(\mathcal{I})} }
     &\stackrel{(a)}{\leq} \sum_{\mathcal{I} \in \mathcal{P}(D)} \sum_{t = 1}^{T}  \frac{ \mathbbm{1}\{\hat{\mathcal{I}}_{t} = \mathcal{I}\} }{ \sqrt{n_{t}(\mathcal{I})} } \\ \nonumber 
     &\leq \sum_{\mathcal{I} \in \mathcal{P}(D)} \sum_{k = 1}^{n_{T}(\mathcal{I})}  \frac{ 1 }{ \sqrt{k} } \\ \nonumber 
     &\leq \sum_{\mathcal{I} \in \mathcal{P}(D)} 2 \sqrt{n_{T}(\mathcal{I})} \\
     &\stackrel{(b)}{\leq} 2 \sqrt{|\mathcal{P}(D)| T},
\end{align}
where $(a)$ holds since $n_{t}(\mathcal{I}) \leq N_{t}(\mathcal{I})$, $\forall \mathcal{I} \in \mathcal{P}(\mathcal{D})$ and $(b)$ follows from Jensen's inequality and the fact that $\sum_{\mathcal{I} \in \mathcal{P}(D)} n_{T}(\mathcal{I}) = T$. Thus, with probability at least $1-\delta$, $\Delta_{1}$ is bounded as
\begin{align} 
\label{eq:second-delta-1}
     &\sum_{t = 1}^{T} \sum_{\boldsymbol{\psi} \in \boldsymbol{\Psi}^{+}(\hat{\mathcal{I}}_{t})} \left[\tilde{p}_{t}(\boldsymbol{\psi}) - p(\boldsymbol{\psi})\right] \tilde{g}_{t}(\hat{h}_{t}(\boldsymbol{\psi}), \boldsymbol{\psi})
     \leq O\left( \sqrt{ \boldsymbol{\Psi}_{tot} \log{(2 T |\mathcal{P}(\mathcal{D})|/\delta)} |\mathcal{P}(D)| T} \right).
\end{align}

It remains to bound the term $\Delta_{2}$. Let $\boldsymbol{e}_{\boldsymbol{\psi}}$ be the unit vector with dimension $|\Psi^{+}(\hat{\mathcal{I}}_{k})|$, where the component corresponding to the state $\boldsymbol{\psi}$ is $1$ and other components are $0$. We rewrite $\Delta_{2}$ as
% Let $\boldsymbol{e}_{\boldsymbol{\psi}_{t}}$ be the $\boldsymbol{\psi}_{t}$-th standard unit vector in $\mathbb{R}^{|\boldsymbol{\Psi}^{+}(\hat{\mathcal{I}}_{t})|}$.
%
\begin{align} \nonumber
\label{eq:first-delta-2}
    \sum_{t = 1}^{T} \sum_{\boldsymbol{\psi} \in \boldsymbol{\Psi}^{+}(\hat{\mathcal{I}}_{t})} &p(\boldsymbol{\psi}) \left[\tilde{g}_{t}(\hat{h}_{t}(\boldsymbol{\psi}), \boldsymbol{\psi}) - g_{t}(\hat{h}_{t}(\boldsymbol{\psi}), \boldsymbol{\psi}) \right] \\
    &= \sum_{t = 1}^{T} {\Big[}
    \langle \mathbf{P}(\hat{\mathcal{I}}_{t}) - \boldsymbol{e}_{\boldsymbol{\psi}_{t}}, \tilde{\mathbf{G}}_{t}(\hat{\mathcal{I}}_{t}) - \mathbf{G}_{t}(\hat{\mathcal{I}}_{t}) \rangle 
    + \langle \boldsymbol{e}_{\boldsymbol{\psi}_{t}}, \tilde{\mathbf{G}}_{t}(\hat{\mathcal{I}}_{t}) - \mathbf{G}_{t}(\hat{\mathcal{I}}_{t}) \rangle {\Big]}.
\end{align}

We continue by bounding the first term in (\ref{eq:first-delta-2}) as follows.
Since the events $\mathcal{E}_{3}$ and $\mathcal{E}_{4}$ hold, it yields that
\begin{align}
    |\hat{r}_{t}(a, \boldsymbol{\psi}) - \bar{r}_{t}(a, \boldsymbol{\psi})| \leq C_{t}(a, \boldsymbol{\psi}; w), \hspace{6mm} \forall a \in \mathcal{A}, \boldsymbol{\psi} \in \boldsymbol{\Psi},
\end{align}
and
\begin{align}
    |\hat{\boldsymbol{c}}_{t}[i] - \bar{\boldsymbol{c}}_{t}[i]| \leq C_{t}(i; w), \hspace{6mm} \forall i \in \mathcal{D}.
\end{align}
%
% defined as $\mathcal{F}_{t} = \sigma(\sigma(\{\mathcal{I}_{\tau}\}_{\tau = 1}^{t}, \{a_{\tau}\}_{\tau = 1}^{t}, \{r_{\tau}\}_{\tau = 1}^{t-1}, \{\boldsymbol{c}_{\tau}[i]\}_{\tau = 1}^{t-1}, \forall i \in \mathcal{I}_{\tau}, \{\boldsymbol{\psi}_{\tau}\}_{\tau = 1}^{t-1})$.
Let $\mathcal{F}_{t}$ be the $\sigma$-algebra generated by $\hat{\mathcal{I}}_{t}$, $a_{t}$, and all the random variables before time $t$ that are revealed to the algorithm. Then, $\boldsymbol{e}_{\boldsymbol{\psi}_{t}}$, $\hat{\mathcal{I}}_{t}$, $\hat{h}_{t}(\boldsymbol{\psi})$, and $\tilde{\mathbf{G}}_{t}(\hat{\mathcal{I}}_{t})$ are $\mathcal{F}_{t}$-measurable and $\mathbb{E}[\boldsymbol{e}_{\boldsymbol{\psi}_{t}} | \mathcal{F}_{t-1}] = \mathbf{P}(\hat{\mathcal{I}}_{t})$. Moreover, $\langle \mathbf{P}(\hat{\mathcal{I}}_{t}) - \boldsymbol{e}_{\boldsymbol{\psi}_{t}}, \tilde{\mathbf{G}}_{t}(\hat{\mathcal{I}}_{t}) - \mathbf{G}_{t}(\hat{\mathcal{I}}_{t}) \rangle$ is a martingale-difference sequence w.r.t. $\mathcal{F}_{t}$. In addition, for $\boldsymbol{\psi} \in \boldsymbol{\Psi}^{+}(\hat{\mathcal{I}}_{t})$, we have
\begin{align} \nonumber
    \tilde{g}_{t}(\hat{h}_{t}(\boldsymbol{\psi}), \boldsymbol{\psi}) &- g_{t}(\hat{h}_{t}(\boldsymbol{\psi}), \boldsymbol{\psi}) \\ \nonumber
    &= \tilde{r}_{t}(\hat{h}_{t}(\boldsymbol{\psi}), \boldsymbol{\psi}) - \sum_{i \in \mathscr{D}(\boldsymbol{\psi})} \tilde{\boldsymbol{c}}_{t}[i] - \bar{r}_{t}(\hat{h}_{t}(\boldsymbol{\psi}), \boldsymbol{\psi}) + \sum_{i \in \mathscr{D}(\boldsymbol{\psi})} \bar{\boldsymbol{c}}_{t}[i] \\  \nonumber
    &=
    \tilde{r}_{t}(\hat{h}_{t}(\boldsymbol{\psi}), \boldsymbol{\psi}) - \hat{r}_{t}(\hat{h}_{t}(\boldsymbol{\psi}), \boldsymbol{\psi})
    + \hat{r}_{t}(\hat{h}_{t}(\boldsymbol{\psi}), \boldsymbol{\psi}) - \bar{r}_{t}(\hat{h}_{t}(\boldsymbol{\psi}), \boldsymbol{\psi}) \\  \nonumber
    &~~~~+\sum_{i \in \mathscr{D}(\boldsymbol{\psi})} \bar{\boldsymbol{c}}_{t}[i] 
    - \sum_{i \in \mathscr{D}(\boldsymbol{\psi})} \hat{\boldsymbol{c}}_{t}[i] + \sum_{i \in \mathscr{D}(\boldsymbol{\psi})} \hat{\boldsymbol{c}}_{t}[i]
    - \sum_{i \in \mathscr{D}(\boldsymbol{\psi})} \tilde{\boldsymbol{c}}_{t}[i] \\  \nonumber
    &\leq |\tilde{r}_{t}(\hat{h}_{t}(\boldsymbol{\psi}), \boldsymbol{\psi}) - \hat{r}_{t}(\hat{h}_{t}(\boldsymbol{\psi}), \boldsymbol{\psi})|
    + |\hat{r}_{t}(\hat{h}_{t}(\boldsymbol{\psi}), \boldsymbol{\psi}) - \bar{r}_{t}(\hat{h}_{t}(\boldsymbol{\psi}), \boldsymbol{\psi})| \\ \nonumber
    &~~~~+\sum_{i \in \mathscr{D}(\boldsymbol{\psi})} | \bar{\boldsymbol{c}}_{t}[i] 
    - \hat{\boldsymbol{c}}_{t}[i]|
    + \sum_{i \in \mathscr{D}(\boldsymbol{\psi})} | \hat{\boldsymbol{c}}_{t}[i]
    - \tilde{\boldsymbol{c}}_{t}[i]| \\ \nonumber
    &\leq 2 C_{t}(\hat{h}_{t}(\boldsymbol{\psi}), \boldsymbol{\psi}; w) + 2 \sum_{i \in \mathscr{D}(\boldsymbol{\psi})} C_{t}(i; w) \\ \label{eq:g-difference-1}
    &\leq 2 \sqrt{ \frac{\log{(T A \boldsymbol{\Psi}_{tot} w/\delta)}}{N_{t}(\hat{h}_{t}(\boldsymbol{\psi}), \boldsymbol{\psi}; w)} } + 2 \hspace{-2mm} \sum_{i \in \mathscr{D}(\boldsymbol{\psi})} \sqrt{ \frac{2\log{(T D w/\delta)}}{N_{t}(i; w)} } \\ \label{eq:g-difference-2}
    &\leq 2 \left[ \sqrt{ \log{(T A \boldsymbol{\Psi}_{tot} w/\delta)} } + D \sqrt{2 \log{(T D w/\delta)} } \right].
\end{align}
Therefore,
%Cauchy-Schwarz inequality
%
\begin{align} 
\label{eq:bounded-variables-in-hoeffding-ineq}
    &\langle \mathbf{P}(\hat{\mathcal{I}}_{t})  - \boldsymbol{e}_{\boldsymbol{\psi}_{t}}, \tilde{\mathbf{G}}_{t}(\hat{\mathcal{I}}_{t})  - \mathbf{G}_{t}(\hat{\mathcal{I}}_{t}) \rangle
    \leq 4 \left[ \sqrt{ \log{(T A \boldsymbol{\Psi}_{tot} w/\delta)} } + D \sqrt{2 \log{(T D w/\delta)} } \right].
    % 2 + 2 \sqrt{?}.
\end{align}
Hence, using the Azuma-Hoeffding inequality stated in Lemma (\ref{lem:Hoeffding}), with probability at least $1 - \delta$, it holds
\begin{align} \nonumber
\label{eq:hoeffding-ineq-applied}
    \sum_{t = 1}^{T} \langle \mathbf{P}(\hat{\mathcal{I}}_{t}) - \boldsymbol{e}_{\boldsymbol{\psi}_{t}}, &\tilde{\mathbf{G}}_{t}(\hat{\mathcal{I}}_{t}) 
    - \mathbf{G}_{t}(\hat{\mathcal{I}}_{t}) \rangle  \\
    &\leq 4 \left[ \sqrt{ \log{(T A \boldsymbol{\Psi}_{tot} w/\delta)} } + D \sqrt{2 \log{(T D w/\delta)} } \right] 
    \sqrt{2 T \log{(1/\delta)}}.
\end{align}

Now, we bound the second term in (\ref{eq:first-delta-2}). Using (\ref{eq:g-difference-1}), we observe that
\begin{align} \nonumber
\label{eq:second-delta-2}
    \sum_{t = 1}^{T} &\langle \boldsymbol{e}_{\boldsymbol{\psi}_{t}}, \tilde{\mathbf{G}}_{t}(\hat{\mathcal{I}}_{t}) 
    - \mathbf{G}_{t}(\hat{\mathcal{I}}_{t}) \rangle \\ \nonumber
    &= \sum_{t = 1}^{T} \sum_{\boldsymbol{\psi} \in \boldsymbol{\Psi}}^{} \mathbbm{1}\{\boldsymbol{\psi}_{t} = \boldsymbol{\psi}\} (\tilde{g}_{t}(\hat{h}_{t}(\boldsymbol{\psi}), \boldsymbol{\psi}) - g_{t}(\hat{h}_{t}(\boldsymbol{\psi}), \boldsymbol{\psi})) \\ \nonumber
    &\leq 2 \sum_{t = 1}^{T} \sum_{\boldsymbol{\psi} \in \boldsymbol{\Psi}}^{} \mathbbm{1}\{\boldsymbol{\psi}_{t} = \boldsymbol{\psi}\} {\Bigg[} \sqrt{ \frac{\log{(T A \boldsymbol{\Psi}_{tot} w/\delta)}}{N_{t}(\hat{h}_{t}(\boldsymbol{\psi}), \boldsymbol{\psi}; w)} } 
    + \sum_{i \in \mathscr{D}(\boldsymbol{\psi})} \sqrt{ \frac{2\log{(T D w/\delta)}}{N_{t}(i; w)} } {\Bigg]} \\ \nonumber
    &= 2 \sqrt{\log{(T A \boldsymbol{\Psi}_{tot} w/\delta)}}  \left[ \underbrace{\sum_{t = 1}^{T} \sum_{\boldsymbol{\psi} \in \boldsymbol{\Psi}}^{} \frac{\mathbbm{1}\{\boldsymbol{\psi}_{t} = \boldsymbol{\psi}\}}{\sqrt{N_{t}(\hat{h}_{t}(\boldsymbol{\psi}), \boldsymbol{\psi}; w)}}}_{\alpha} \right]  \\ 
    &\hspace{10mm}+ 2 \sqrt{2\log{(T D w/\delta)}}
     \left[ \underbrace{\sum_{t=1}^{T} \sum_{\boldsymbol{\psi} \in \boldsymbol{\Psi}}^{} \mathbbm{1}\{\boldsymbol{\psi}_{t} = \boldsymbol{\psi}\}  \sum_{i \in \mathscr{D}(\boldsymbol{\psi})} \frac{1}{\sqrt{N_{t}(i; w)}}}_{\beta} \right].
\end{align}

For the term $\alpha$, similar to \cite{Vernade20:ORC}, we split the time horizon into intervals $I_{\ell} = [\ell w - w, \ell w - 1]$ of length $w$. For any interval $I_{\ell}$ and any $t \in I_{\ell}$, let $N_{t}(a, \boldsymbol{\psi}; \ell)$ and $N_{t}(i; \ell)$ represent the number of times the pair $(a,\boldsymbol{\psi})$ was chosen in $[\ell w - w, t - 1]$ and the number of times the feature $i$ was selected in $[\ell w - w, t - 1]$, respectively. If no such pair $(a,\boldsymbol{\psi})$ and feature $i$ was chosen in $[\ell w - w, t - 1]$, we set $N_{t}(a, \boldsymbol{\psi}; \ell)$ and $N_{t}(i; \ell)$ equal to $1$, respectively. We observe that $N_{t}(a, \boldsymbol{\psi}; \ell) \leq N_{t}(a, \boldsymbol{\psi}; w)$ and $N_{t}(i; \ell) \leq N_{t}(i; w)$. Therefore,
\begin{align} \nonumber
\label{eq:alpha}
    \sum_{t=1}^{T} \sum_{\boldsymbol{\psi} \in \boldsymbol{\Psi}}^{} \frac{\mathbbm{1}\{\boldsymbol{\psi}_{t} = \boldsymbol{\psi}\}}{\sqrt{N_{t}(\hat{h}_{t}(\boldsymbol{\psi}), \boldsymbol{\psi}; w)}}
    &\leq \sum_{\ell=1}^{\left\lceil{\frac{T}{w}}\right\rceil} \sum_{t \in I_{\ell}}^{} \sum_{a \in \mathcal{A}}^{} \sum_{\boldsymbol{\psi} \in \boldsymbol{\Psi}}^{} \frac{\mathbbm{1}\{\boldsymbol{\psi}_{t} = \boldsymbol{\psi} \& a_{t} = a\}}{\sqrt{N_{t}(a, \boldsymbol{\psi}; w)}} \\ \nonumber
    &= \sum_{\ell=1}^{\left\lceil{\frac{T}{w}}\right\rceil} \sum_{a \in \mathcal{A}}^{} \sum_{\boldsymbol{\psi} \in \boldsymbol{\Psi}}^{} \sum_{t \in I_{\ell}}^{} \frac{\mathbbm{1}\{\boldsymbol{\psi}_{t} = \boldsymbol{\psi} \& a_{t} = a\}}{\sqrt{N_{t}(a, \boldsymbol{\psi}; \ell)}} \\ \nonumber
    &\stackrel{(a)}{\leq} \sum_{\ell=1}^{\left\lceil{\frac{T}{w}}\right\rceil} 2 \sqrt{A \boldsymbol{\Psi}_{tot} (w+1)} \\
    &\stackrel{(b)}{\leq} 2 (\frac{T}{w} + 1) \sqrt{A \boldsymbol{\Psi}_{tot} (w+1)},
\end{align}
where $(a)$ holds because of the inequality $\sum_{i=1}^{v} \frac{1}{\sqrt{i}} \leq 2 (\sqrt{v + 1} - 1)$ and due to the fact that the last sum reaches its highest value when each pair $(a, \boldsymbol{\psi})$ is selected $\left\lfloor \frac{w}{A \boldsymbol{\Psi}_{tot}}\right\rfloor \leq \frac{w}{A \boldsymbol{\Psi}_{tot}}$ times in the interval  $I_{\ell}$. Moreover, $(b)$ holds since the number of intervals $I_{\ell}$ is at most $\left\lceil{\frac{T}{w}}\right\rceil \leq \frac{T}{w} + 1$.
% , up to rounding errors,
% Note that, if the pair $(a, \boldsymbol{\psi})$, for some $a$ and $\boldsymbol{\psi}$, is not selected, still we have a term $\frac{1}{\sqrt{1}}$ added to the sum. For those pairs $(a, \boldsymbol{\psi})$ that are selected, the term $\frac{1}{\sqrt{1}}$ is added twice.

For the term $\beta$, we have
\begin{align} \nonumber
\label{eq:beta}
    \sum_{t = 1}^{T} \sum_{\boldsymbol{\psi} \in \boldsymbol{\Psi}}^{} \mathbbm{1}\{\boldsymbol{\psi}_{t} = \boldsymbol{\psi}\} \sum_{i \in \mathscr{D}(\boldsymbol{\psi})} \frac{1}{\sqrt{N_{t}(i; w)}}
    &= \sum_{t = 1}^{T} \sum_{\boldsymbol{\psi} \in \boldsymbol{\Psi}}^{} \mathbbm{1}\{\boldsymbol{\psi}_{t} = \boldsymbol{\psi}\} \sum_{i \in \mathcal{D}} \frac{\mathbbm{1}\{i \in  \mathscr{D}(\boldsymbol{\psi})\}}{\sqrt{N_{t}(i; w)}} \\ \nonumber
    % &= \sum_{t = 1}^{T} \sum_{\boldsymbol{\psi} \in \boldsymbol{\Psi}}^{} \sum_{i \in \mathcal{D}} \frac{\mathbbm{1}\{\boldsymbol{\psi}_{t} = \boldsymbol{\psi} \& i \in \mathscr{D}(\boldsymbol{\psi})\}}{\sqrt{N_{t}(i; w)}} \\ \nonumber
    &\stackrel{(a)}{\leq} \sum_{t = 1}^{T} \sum_{i \in \mathcal{D}} \frac{\mathbbm{1}\{i \in  \hat{\mathcal{I}}_{t}\}}{\sqrt{N_{t}(i; w)}} \\ \nonumber
    &\leq \sum_{\ell=1}^{\left\lceil{\frac{T}{w}}\right\rceil}  \sum_{i \in \mathcal{D}} \sum_{t \in I_{\ell}}^{} \frac{\mathbbm{1}\{i \in  \hat{\mathcal{I}}_{t}\}}{\sqrt{N_{t}(i; \ell)}} \\
    &\stackrel{(b)}{\leq} 2 (\frac{T}{w} + 1) D \sqrt{w+1},
\end{align}
where $(a)$ holds because at each time $t$, regardless of the agent's choice of observation set, at most $D$ features' states can be observed. Moreover, $(b)$ follows by a similar reasoning as the one given for (\ref{eq:alpha}); the only difference here is that, the agent can choose to observe more than one feature's state at each time $t$. This means that, unlike the counts $N_{t}(a, \boldsymbol{\psi}; \ell)$, the counts $N_{t}(i; \ell)$ can be increased by $1$ for more than one feature $i$ at each time $t$. Thus, we consider the worst case where $D$ features' states are observed at each time of play.

Therefore, by using (\ref{eq:alpha}) and (\ref{eq:beta}) in (\ref{eq:second-delta-2}), and combining the results with (\ref{eq:hoeffding-ineq-applied}), with probability at least $1-3\delta$, the following bound holds for $\Delta_{2}$.
\begin{align} \nonumber
\label{eq:third-delta-2}
    \sum_{t = 1}^{T} &\sum_{\boldsymbol{\psi} \in \boldsymbol{\Psi}^{+}(\hat{\mathcal{I}}_{t})} p(\boldsymbol{\psi}) \left[\tilde{g}_{t}(\hat{h}_{t}(\boldsymbol{\psi}), \boldsymbol{\psi}) - g_{t}(\hat{h}_{t}(\boldsymbol{\psi}), \boldsymbol{\psi})\right] \\ \nonumber
    &\leq O {\Bigg(} T  {\Big(} \sqrt{ \frac{A \boldsymbol{\Psi}_{tot} \log{(T A \boldsymbol{\Psi}_{tot} w/\delta)}}{w} } + D \sqrt{ \frac{\log{(T D w/\delta)}}{w} } {\Big)} \\ 
    &\hspace{8mm}+ \sqrt{T \log{(1/\delta)}} {\Big(} \sqrt{  A \boldsymbol{\Psi}_{tot} \log{(T A \boldsymbol{\Psi}_{tot} w/\delta)} }
    + D \sqrt{ \log{(T D w/\delta)} } {\Big)} {\Bigg)}.
\end{align}
We conclude the proof by combining (\ref{eq:second-delta-1}) and (\ref{eq:third-delta-2}).
% that, with probability at least $1-3\delta$, the regret of our proposed policy stated in (\ref{eq:final-regret-bound}) holds.
%
\end{proof}
%-------------------------------------

%-------------------------------------> Proof of Theorem 2
% \ref{app:TheoremTwoProof}
% \subsubsection{Proof of Theorem \ref{thm:regret-nonstationary}}
\subsubsection{Proof of Theorem \ref{thm:regret-nonstationary}} %\ref{thm:regret-nonstationary} 2
\label{app:TheoremTwoProof}
%
%------------------------------------->
\begin{proof}
For any positive $T$, define $\Gamma(w)$ as
\begin{align} \nonumber
\Gamma(w) = {\Big\{}t \in \{1, \dots, T\} ~{\Big|}~ 
\bar{r}_{\tau}(a, \boldsymbol{\psi}) = \bar{r}_{t}(a, \boldsymbol{\psi}) ~\&~ \bar{\boldsymbol{c}}_{\tau}[i] = \bar{\boldsymbol{c}}_{t}[i], &\forall a \in \mathcal{A}, \forall \boldsymbol{\psi} \in \boldsymbol{\Psi}, \\
&\forall i \in \mathcal{D}, \forall \tau \hspace{1mm} \text{s.t.} \hspace{1mm} t - w < \tau \leq t {\Big\}}.
\end{align}
In our problem, there are $\Upsilon_{T} + 1$ stationary periods. We add the first and last round to the change points and denote them by $1 = \tau_{0}, \dots, \tau_{\Upsilon_{T} + 1} = T$. Moreover, consider the events $\mathcal{E}_{1}$, $\mathcal{E}_{3}$, and $\mathcal{E}_{4}$, defined in \ref{eq:event1}, \ref{eq:event3}, and \ref{eq:event4}, respectively. We redefne these events for $t \in \Gamma(w)$ instead of $t \in [T]$ to include the time instances belonging only to $\Gamma(w)$, and denote the resulting events by $\mathcal{E}_{1}(w)$, $\mathcal{E}_{3}(w)$, and $\mathcal{E}_{4}(w)$, respectively. By the same reasoning as in Lemma \ref{lem:events}, it holds that $\mathbb{P}[\bar{\mathcal{E}}_{1}(w) \cup \bar{\mathcal{E}}_{2} \cup \bar{\mathcal{E}}_{3}(w) \cup \bar{\mathcal{E}}_{4}(w)] \leq 3\delta$.

Now, we assume that the events $\mathcal{E}_{1}(w)$, $\mathcal{E}_{2}$, $\mathcal{E}_{3}(w)$, and $\mathcal{E}_{4}(w)$ hold and follow the same reasoning as in the proof of Theorem 1. This results in the following regret bound that holds with probability at least $1-3\delta$. %\ref{thm:regret-stationary}
\begin{align} 
\label{eq:regret-decomposition-final}
    \mathcal{R}_{T}(\Pi) 
    &\leq w \Upsilon_{T} + \sum_{t = 1}^{T} \langle \tilde{\mathbf{P}}_{t}(\hat{\mathcal{I}}_{t}) - \mathbf{P}(\hat{\mathcal{I}}_{t}), \tilde{\mathbf{G}}_{t}(\hat{\mathcal{I}}_{t}) \rangle
    + \sum_{i = 0}^{\Upsilon_{T}} \sum_{\tau_{i} + w}^{\tau_{i+1} - 1} \langle \mathbf{P}, \tilde{\mathbf{G}}_{t} - \mathbf{G}_{t} \rangle.
\end{align}
The last term can be bounded similar to (\ref{eq:third-delta-2}) in the proof of Theorem \ref{thm:regret-stationary}. Therefore, %\ref{thm:regret-stationary} 1
\begin{align} \nonumber
\label{eq:decomposition-nonstationary}
    &\sum_{i = 0}^{\Upsilon_{T}} \sum_{\tau_{i} + w}^{\tau_{i+1} - 1} \langle \mathbf{P}, \tilde{\mathbf{G}}_{t} - \mathbf{G}_{t} \rangle \\ \nonumber
    &\leq \sum_{i = 0}^{\Upsilon_{T}} ~O {\Bigg(} (\tau_{i+1} - \tau_{i}) {\Big(} \sqrt{ \frac{A \boldsymbol{\Psi}_{tot} \log{(T A \boldsymbol{\Psi}_{tot} w/\delta)}}{w} }
    + D \sqrt{ \frac{\log{(T D w/\delta)}}{w} } {\Big)} \\ \nonumber
    &\hspace{8mm}+ \sqrt{(\tau_{i+1} - \tau_{i})  \log{(1/\delta)}} 
    {\Big(} \sqrt{  A \boldsymbol{\Psi}_{tot} \log{(T A \boldsymbol{\Psi}_{tot} w/\delta)} } + D \sqrt{ \log{(T D w/\delta)} } {\Big)} {\Bigg)} \\ \nonumber
    &\leq O {\Bigg(} T {\Big(} \sqrt{ \frac{A \boldsymbol{\Psi}_{tot} \log{(T A \boldsymbol{\Psi}_{tot} w/\delta)}}{w} } + D \sqrt{ \frac{\log{(T D w/\delta)}}{w} } {\Big)} \\ 
    &\hspace{8mm}+ \sqrt{\Upsilon_{T} T  \log{(1/\delta)}} {\Big(} \sqrt{  A \boldsymbol{\Psi}_{tot} \log{(T A \boldsymbol{\Psi}_{tot} w/\delta)} } 
    + D \sqrt{ \log{(T D w/\delta)} } {\Big)} {\Bigg)},
\end{align}
where the last inequality follows from Jensen's inequality and the fact that $\sum_{i = 0}^{\Upsilon_{T}} (\tau_{i+1} - \tau_{i}) = T$.
Thus, summarizing the above results, and by using (\ref{eq:second-delta-1}) to bound the second term in (\ref{eq:regret-decomposition-final}), we conclude the proof.
\end{proof}
%-------------------------------------

%------------------------------------->
% \begin{figure}[t!]
%     \centering
%     \includegraphics[width=0.70\textwidth]{Images/nursery_rewards.pdf}
%     \caption{Evolution of the mean reward for each arm.}
%     \label{fig:MeanRewards}
% \end{figure}
%-------------------------------------
%------------------------------------->
% \begin{figure}[b!]
%     \centering
%     \includegraphics[width=0.70\textwidth]{Images/nursery_costs.pdf}
%     \caption{Evolution of the mean cost for each feature.}
%     \label{fig:MeanCosts}
% \end{figure}
%-------------------------------------
%-------------------------------------> Additional Info for Experiments
% \subsection{Additional Information on Experimental Setup}
\subsection{Additional Information on Experimental Setup}
% DDITIONAL INFORMATION ON EXPERIMENTAL SETUP}
\label{app:AddInfo}
%
% \textbf{Experimental Setup:~}
%------------------------------------->
\begin{figure}[b!]
\begin{center}
    \begin{subfigure}[b]{0.56\textwidth}
    \centering
    \hspace{-5mm}
        \includegraphics[width = 0.98\textwidth]{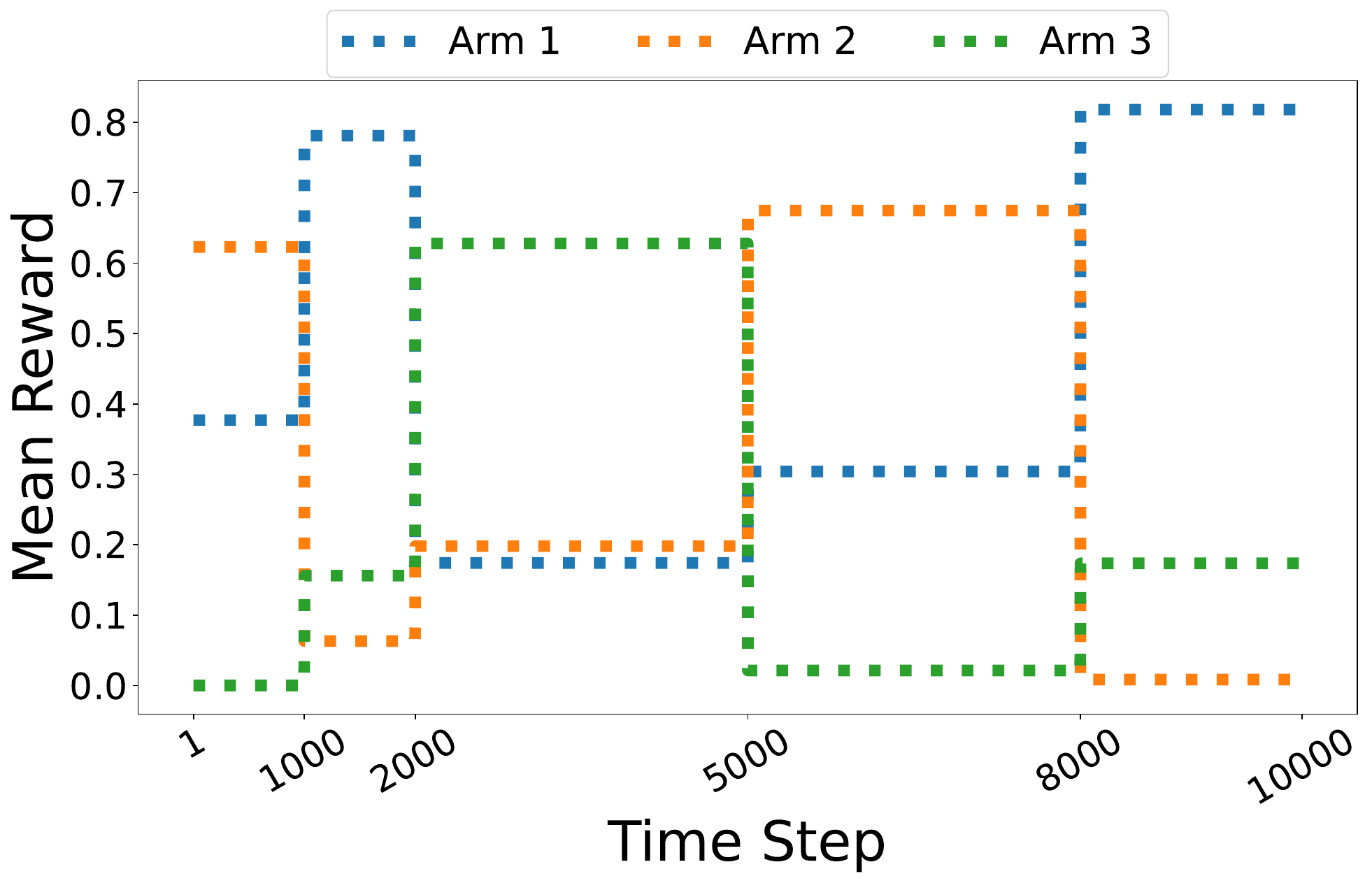}
        \caption{Evolution of the mean reward for each arm.}
        \label{fig:MeanRewards}
    \end{subfigure} \hfill 
    \begin{subfigure}[b]{0.61\textwidth}
    \centering
    \vspace{5mm}
        \includegraphics[width = 0.98\textwidth]{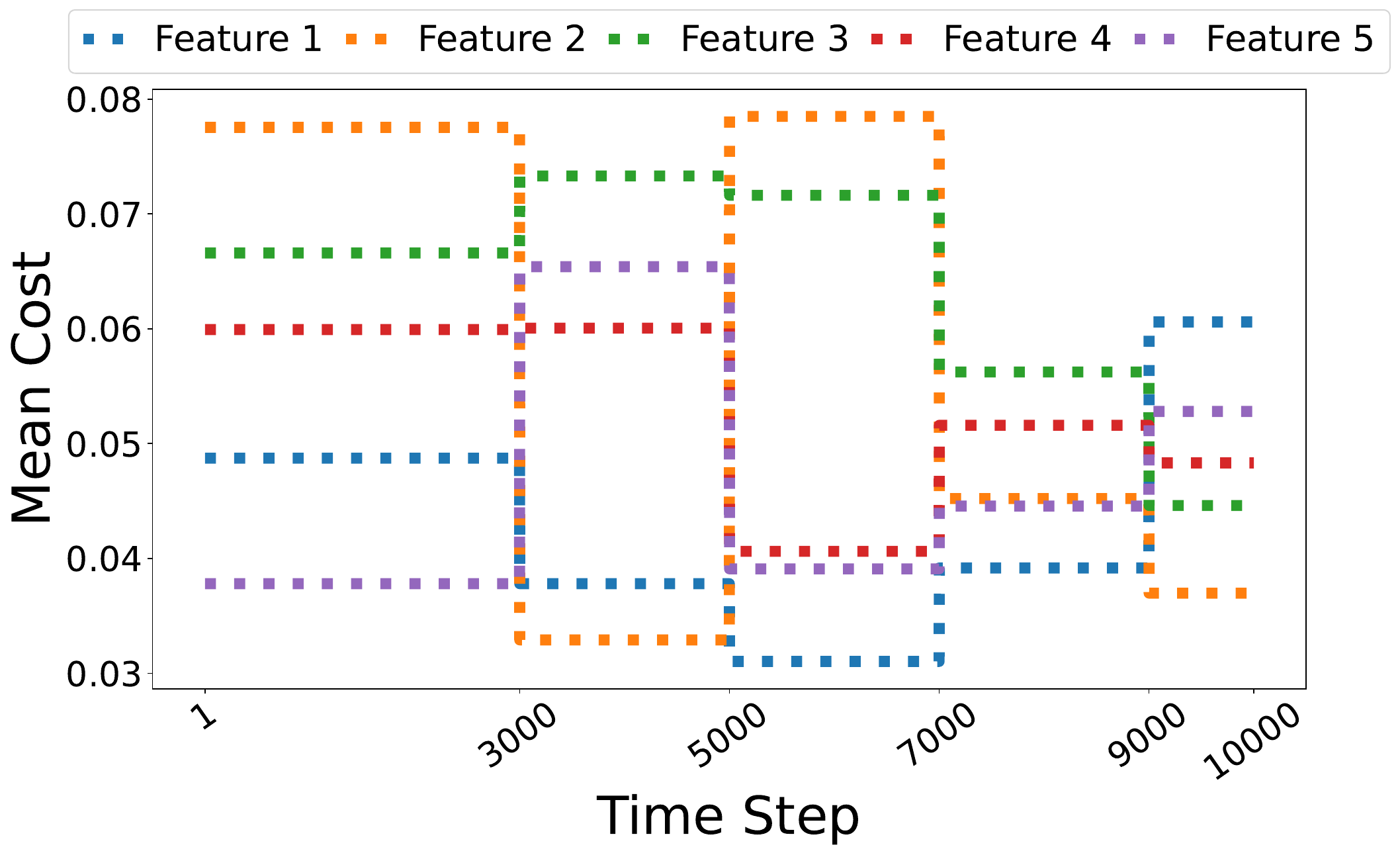}
        % \vspace{1mm}
        \caption{Evolution of the mean cost for each feature.}
        \label{fig:MeanCosts}
    \end{subfigure}
\end{center}
\caption{Settings of mean rewards and mean costs.}
% \caption{(a) Evolution of the mean reward for each arm. (b) Evolution of the mean cost for each feature.}
\label{Fig:AllMeanValues}
\end{figure}
%-------------------------------------

\textbf{Fig. \ref{fig:MeanRewards}} and \textbf{\ref{fig:MeanCosts}} depict the changes in the mean reward for each arm and in the mean cost for each feature, respectively. As we see, the change points in mean rewards and mean costs are not necessarily identical. 

The parameters of benchmark policies in our experiment are listed in \textbf{Table \ref{table:PolicyParams}}. As mentioned before, to tune the parameters of NCC-UCRL2 and PS-LinUCB, we consider $2$ change points in mean rewards at times $\{1000, 2000\}$, but no change points in mean costs. For NCC-UCRL2, We simultaneously tuned $w$ and $\delta$ by performing a grid search over the sets $\{100, 250, 350, 500, 600, \newline 800, 900, 1000, 1250, 1500, 1750\}$ and $\{0.002, 0.006, 0.009, 0.01, 0.02, 0.03, 0.04, 0.05, 0.07, 0.08, \newline 0.1, 0.2, 0.3, 0.5, 0.7, 0.8, 0.9\}$, respectively. To that end, we ran the algorithm with each pair of parameters for three repetitions and chose parameters that resulted in the highest average gain.
%  Moreover, 

%-------------------------------------> Table
\renewcommand{\arraystretch}{1.1}
\renewcommand{\tabcolsep}{1.5mm}
%Single-Column Version
% \begin{table}[b!]
% \caption{Parameters of the different policies in the experiment.}
% \label{table:PolicyParams}
% {%\small
% % \centering
% \begin{center}
% \begin{tabular}{c|c|c|c|c|c|c}
%     \cline{2-7}
%     %\hline
%     \multicolumn{1}{c|}{} &
%     \multicolumn{6}{c}{Policy Setting} \\
%     \hline
%     Policy & Sim-OOS & PS-LinUCB & LinUCB & UCB1 & $\varepsilon$-Greedy & NCC-UCRL2 \\
%     \hline
%     \multirow{3}{*}{Parameters} & $\delta = 0.8$ & $\alpha = 0.7$ & $\alpha = 0.5$ & $\alpha = 0.6$ & $\varepsilon = 0.03$ & $w = 250$\\
%     & & $\omega = 100$  &  & & & $\delta = 0.04$\\
%     & & $\delta = 0.05$  &  & & & \\
%     \hline
%     \end{tabular}
%     \end{center}
%     }     
% \end{table}
%Double-Column Version
\begin{table}[t!]
\caption{Parameters of the different policies in the experiment.}
\label{table:PolicyParams}
{%\footnotezie
% \centering
\begin{center}
\begin{tabular}{l|c c c} %{0.1\textwidth}
    % \cline{2-7}
    \hline
    Policy & \multicolumn{3}{c}{Parameters} \\
    \hline
    Sim-OOS & $\delta = 0.8$ & & \\
    \hline
    PS-LinUCB & $\alpha = 0.7$ & $\omega = 100$ & $\delta = 0.05$ \\
    \hline
    LinUCB & $\alpha = 0.5$ &  & \\
    \hline
    UCB1 & $\alpha = 0.6$ & & \\
    \hline
    $\varepsilon$-Greedy & $\varepsilon = 0.03$ & & \\
    \hline
    NCC-UCRL2 & $w = 250$ & $\delta = 0.04$ & \\
    \hline
    \end{tabular}
    \end{center}
    }     
\end{table}
%-------------------------------------

%-------------------------------------> section Additional Experiments
\subsection{Additional Experiments}
% DDITIONAL EXPERIMENTS}
\label{app:AddExp}
%
%------------------------------------->
\begin{figure}[b!]
    \centering
    \includegraphics[width=0.64\textwidth]{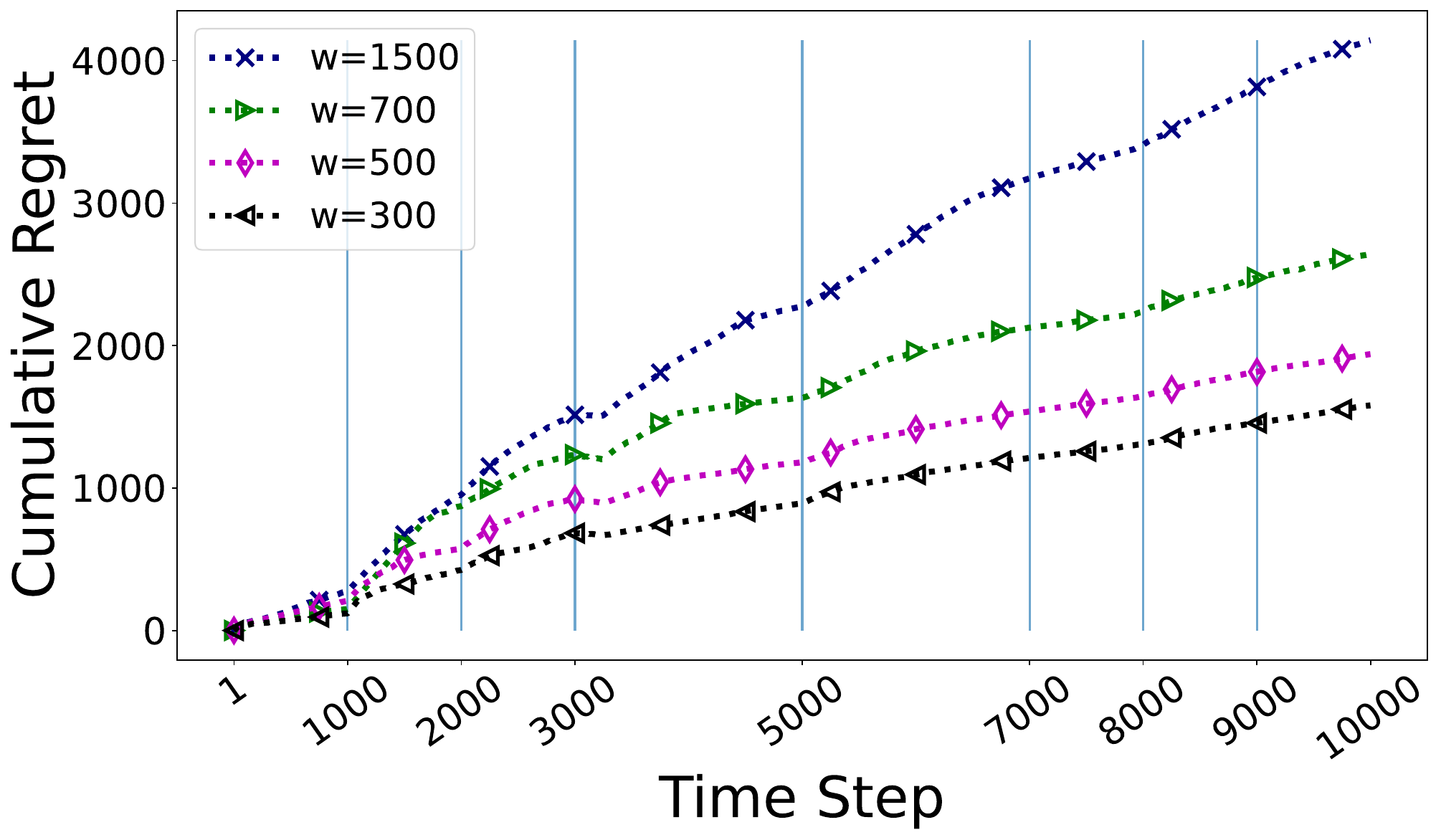}
    \caption{Cumulative regret of NCC-UCRL2 for different window parameters $w$.}
    \label{fig:Effectofw}
\end{figure}
%-------------------------------------
% \subsubsection{Effect of Window Length $w$} 
\textbf{Effect of Window Length $w$:}
Choosing the right window parameter $w$ is crucial to ensure that the NCC-UCRL2 algorithm promptly adjusts the decision-making strategy after sudden changes while maintaining a good performance during stationary periods. The window size $w$ can be chosen based on the change frequency. A smaller $w$ allows for faster adaptation but reduces the performance during stationary periods due to exploiting fewer relevant data samples. In an environment with infrequent change points, a larger $w$ is more suitable as it results in a better performance between change points, although the algorithm requires more storage space. \textbf{Fig. \ref{fig:Effectofw}} illustrates the trend of cumulative regret of our algorithm when running on the nursery dataset with different window parameters $w$. Based on our simulation's setting, we see that NCC-UCRL2 with smaller window sizes (around $300$) results in a much lower regret (e.g., compared to values more than $700$).

% \subsubsection{Accuracy:} 
\textbf{Accuracy:} 
To further analyze the performance of our algorithm, we define \textit{accuracy} for the model based on the number of state observations. With $\ell$ observations, the accuracy yields $\left(\sum_{j=0}^{\ell}{\sum_{t=1}^{T}{r_t \mathbbm{1}\{|\mathcal{I}_t| = j \}}}\right) / 
\left(\sum_{j=0}^{\ell}{\sum_{t=1}^{T}{\mathbbm{1}\{|\mathcal{I}_t| = j \}}}\right)
$. We use the term accuracy since, in our experiment, a reward of $1$ implies the correct classification of a nursery application. \cite{Atan21:DDO} perform a similar analysis for Sim-OOS. Therefore, we plot the accuracy of NCC-UCRL2 and Sim-OOS for a different number of observations in \textbf{Fig. \ref{fig:Accuracy}}, as these are the only algorithms that implement feature selection. For fewer observations, the accuracy of Sim-OOS is close to that of NCC-UCRL2, while NCC-UCRL2 achieves a higher accuracy as the number of observations increases. This again shows the importance of learning the optimal observations and demonstrates the superiority of our method.
%------------------------------------->
\begin{figure}[t!]
    \centering
    \includegraphics[width=0.7\textwidth]{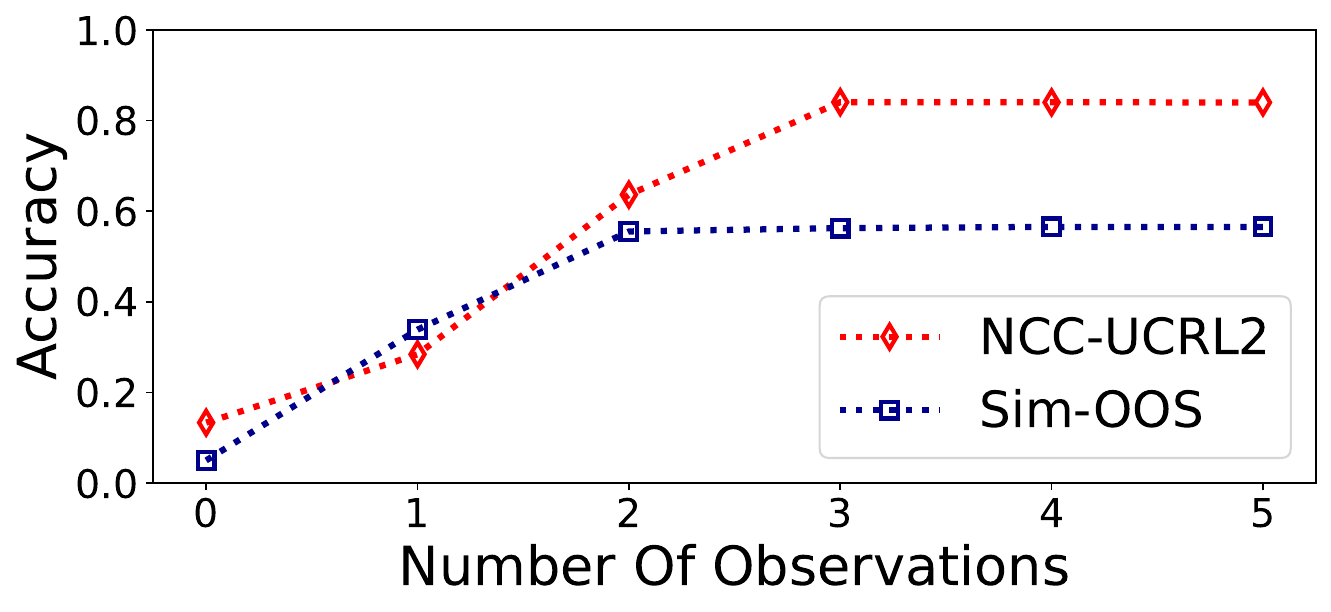}
    \caption{Accuracy for different number of observations.}
    \label{fig:Accuracy}
\end{figure}
%-------------------------------------
%
%-------------------------------------> References
\bibliographystyle{IEEEbib}
\bibliography{references}
%-------------------------------------
% \newpage
\end{document}